\documentclass[10pt,twocolumn,letterpaper]{article}

\usepackage{cvpr}              

\usepackage[dvipsnames]{xcolor}

\usepackage{times}
\usepackage{epsfig}
\usepackage{graphicx}
\usepackage{amsmath}
\usepackage{amssymb}

\usepackage{graphicx}
\usepackage{amsmath}
\usepackage{amssymb}
\usepackage{amsthm}
\usepackage{booktabs}
\usepackage{caption}
\usepackage{multicol}
\usepackage{multirow}
\usepackage{makecell}
\usepackage{comment}
\usepackage{amssymb}
\usepackage{pifont}
\usepackage{enumitem}

\newtheorem{proposition}{Proposition}

\newcommand{\xmark}{\ding{55}}%

\definecolor{cvprblue}{rgb}{0.21,0.49,0.74}
\usepackage[pagebackref,breaklinks,colorlinks,citecolor=cvprblue]{hyperref}

\def\paperID{109} 
\def\confName{3DV\xspace}
\def\confYear{2024\xspace}

\title{Handbook on Leveraging Lines for Two-View Relative Pose Estimation
}

\author{Petr Hruby${}^1$
\and
Shaohui Liu${}^1$
\and
Rémi Pautrat${}^1$
\and
Marc Pollefeys${}^{1, 2}$
\and
Daniel Barath${}^1$
\and
${}^1$ \normalsize{Department of Computer Science, ETH Zurich}
\and
${}^2$ \normalsize{Microsoft Mixed Reality and AI Zurich lab}
}

\usepackage[capitalize]{cleveref}
\crefname{section}{Sec.}{Secs.}
\Crefname{section}{Section}{Sections}
\Crefname{table}{Table}{Tables}
\crefname{table}{Tab.}{Tabs.}

\newcommand*\matr{\mathbf}

\newcommand{\Proj}{\matr{P}}
\newcommand{\Hom}{\matr{H}}

\newcommand{\Fund}{\matr{F}}
\newcommand{\Ess}{\matr{E}}
\newcommand{\Intrinsic}{\matr{K}}
\newcommand*\Point{\mathbf{p}}
\newcommand*\Worldpoint{\mathbf{X}}
\newcommand*\Qoint{\mathbf{q}}
\newcommand{\Rot}{\matr{R}}

\newcommand{\Tran}{\mathbf{t}}

\newcommand{\Normline}{\mathbf{n}}
\newcommand{\Line}{\mathbf{l}}
\newcommand{\Worldline}{\mathbf{L}}
\newcommand{\RR}{\mathbb{R}}
\newcommand{\Vanishing}{\mathbf{v}}
\newcommand{\Dir}{\mathbf{d}}
\newcommand{\Plane}{\mathbf{\Pi}}

\begin{document}
\maketitle

\begin{abstract}
   We propose an approach for estimating the relative pose between calibrated image pairs by jointly exploiting points, lines, and their coincidences in a hybrid manner. 
We investigate all possible configurations where these data modalities can be used together and review the minimal solvers available in the literature. 
Our hybrid framework combines the advantages of all configurations, enabling robust and accurate estimation in challenging environments. 
In addition, we design a method for jointly estimating multiple vanishing point correspondences in two images, and a bundle adjustment that considers all relevant data modalities. 
Experiments on various indoor and outdoor datasets show that our approach outperforms point-based methods, improving AUC@10$^\circ$ by 1-7 points while running at comparable speeds.
The source code of the solvers and hybrid framework will be made public.

\end{abstract}

\vspace{-10pt}
\section{Introduction} \label{sec:intro}

Estimating the relative pose (\ie, rotation and translation) between an image pair is a fundamental problem both in computer vision and robotics that has numerous real-world applications, \eg, 
in 3D reconstruction \cite{DBLP:journals/cacm/AgarwalFSSCSS11, DBLP:conf/cvpr/BarathMESM21, DBLP:conf/cvpr/HeinlySDF15, DBLP:conf/cvpr/SchonbergerF16, DBLP:conf/cvpr/ZhuZZSFTQ18, DBLP:journals/ijcv/SnavelySS08}, 
visual localization \cite{DBLP:journals/ijrr/LynenZABHPSS20, DBLP:conf/eccv/PanekKS22, DBLP:conf/eccv/SattlerLK12, DBLP:conf/cvpr/SattlerMTTHSSOP18}, 
simultaneous localization and mapping \cite{DBLP:journals/corr/DeToneMR17, superpoint, DBLP:conf/eccv/EngelSC14, DBLP:journals/trob/Mur-ArtalMT15}, 
multi-view stereo \cite{DBLP:conf/iccv/ChenHXS19, DBLP:conf/cvpr/FurukawaCSS10, DBLP:journals/ftcgv/FurukawaH15, DBLP:conf/nips/KarHM17}, 
and visual odometry \cite{DBLP:conf/cvpr/NisterNB04, DBLP:journals/jfr/NisterNB06}.
In this paper, we focus on estimating the relative pose in a hybrid manner, jointly from 2D line and point correspondences and their coincidences. 
This allows for being robust to various indoor and outdoor scene characteristics, \eg, low-textured areas where lines tend to be more distinctive than points.

\begin{figure}
    \centering
    \includegraphics[width=\columnwidth,trim={0 25 200 5},clip]{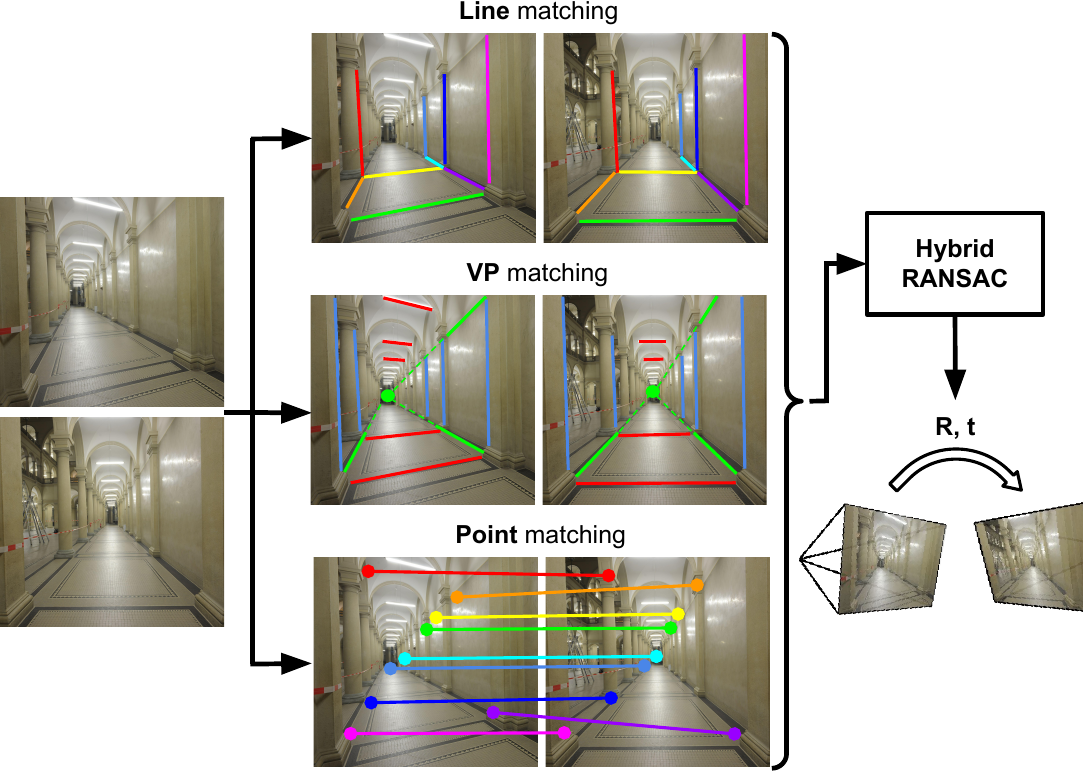}
    \caption{\textbf{Relative pose from points and lines.} We present all configurations to exploit point, line, and vanishing point correspondences for estimating the relative pose of two calibrated images. By combining the configurations within a hybrid RANSAC \cite{fischler1987} framework our approach can handle typical failure cases of the widely used 5-point solver~\cite{DBLP:conf/cvpr/Nister03}, \eg in low textured areas.}
    \label{fig:teaser}
\end{figure}

The traditional approach for estimating relative pose in two images involves detecting \cite{SIFT2004, DBLP:conf/nips/RevaudSHW19, DBLP:conf/nips/TyszkiewiczFT20, superpoint} and matching \cite{sarlin20superglue} local features to form tentative point correspondences. 
They are then fed into a robust estimator, such as RANSAC or one of its variants \cite{chum2003locally, DBLP:conf/cvpr/BarathNIM20, barath2018graph, DBLP:journals/pami/RaguramCPMF13, DBLP:conf/iccv/IvashechkinBM21}, to simultaneously find the sought relative pose and the matches consistent with it.
Although this point-based approach is still widely used and forms the cornerstone of many vision applications, it has certain weaknesses that deteriorate its accuracy in scenes dominated by homogeneous or repetitive regions.
This poses a challenge, especially in indoor scenes \cite{7scenes, dai2017scannet, Schops_2017_eth3d} that often contain low-textured areas, \eg walls, preventing to find distinctive features.
Repetitive structures also frequently appear in man-made environments, \eg windows on a facade, breaking the visual descriptor-based feature matching due to the implied ambiguity.  

Several alternative approaches have been proposed for relative pose estimation, including ones leveraging optical flow~\cite{DBLP:journals/ai/HornS81, DBLP:conf/iccv/DosovitskiyFIHH15}, or using features that contain richer information than simply the point coordinates~\cite{DBLP:journals/tip/BarathH18,barath2022relative}. 
While algorithms based on optical flow are widely used in SLAM pipelines~\cite{DBLP:journals/trob/Mur-ArtalMT15}, they assume a relatively small camera motion. Thus, they are not applicable to general relative pose estimation with cameras moving arbitrarily. 
Other methods exploit rich features, such as affine correspondences, to solve the problem with fewer matches than when using only points. 
This reduces the combinatorics of the robust estimation problem and can often improve both accuracy and runtime. 
However, these methods are also subject to the same weaknesses as point-based approaches in that they require features to be located on salient regions to estimate their affine shape accurately \cite{HesHarAff2004, AffNet2018, Mishkin2015MODS, DBLP:journals/ipol/YuM11}. 

Lines are known to be particularly useful, especially in low-textured areas, and are actively used for 3D localization \cite{abdellali2021l2d2,gao2021pose} or reconstruction using 2D-3D line matches \cite{DBLP:journals/cviu/BartoliS05,zuo2017robust,pumarola2017pl,gomez2019pl,wei2019real,lim2021avoiding,shu2022structure}. 
Also, a growing number of papers investigate their potential when having more than two images~\cite{DBLP:conf/iccv/DuffKLP19, DBLP:conf/eccv/DuffKLP20, DBLP:conf/cvpr/FabbriDFRPTWHGK20, breiding2022line, Geppert2020ECCV, Geppert2021CVPR, DBLP:conf/cvpr/HrubyDLP22}.
However, their use for relative pose estimation in a stereo setup is limited as corresponding 2D lines do not impose explicit constraints on the relative camera pose~\cite{DBLP:conf/iccv/DuffKLP19}. 
There are several works leveraging lines for two-view geometry estimation. 
Guerrero et al.~\cite{guerrero2001lines} estimate a homography from four collinear line correspondences by the well-known direct linear transformation. 
Elqursh et al.~\cite{DBLP:conf/cvpr/ElqurshE11} assumes a triplet of lines to be in a special configuration, allowing to estimate the relative camera rotation decoupled from the translation.

This paper investigates the configurations where points and lines can be used to estimate the relative pose between two calibrated views. 
Even though there are several solvers proposed throughout the years \cite{guerrero2001lines,DBLP:conf/cvpr/ElqurshE11,SalaunMM16} using lines and vanishing points to estimate relative pose, there is no comprehensive overview nor comparison of how such methods can be used in practice. 
We provide a list of the relevant point, vanishing point, and line configurations and review the minimal solvers available in the literature.
Benefiting from this knowledge, we develop a unified framework that simultaneously exploits multiple data modalities in a hybrid manner to provide robust and accurate results even in challenging environments. 
The contributions are: 
\begin{itemize}[noitemsep]
    \item We investigate \textit{all} relevant data configurations, where points, lines, and their coincidences (\eg, vanishing points and junctions) can be used together.
    \item We review the minimal solvers for the configurations available in the literature and provide an overview.
    \item We transform the configurations not available in the literature to the known problems to solve them.
    \item We develop a unified framework that simultaneously benefits from multiple feature types in a hybrid manner for estimating the relative pose.
    \item In addition, we provide proof that the constraint derived from coplanar lines is equivalent to using line junctions while leading to more stable solvers. 
    We propose a joint vanishing point estimation method between two images; and a local optimization algorithm that simultaneously optimizes over all data modalities. 
\end{itemize}
We demonstrate on several public, real-world, and large-scale datasets (both indoor and outdoor) that the proposed approach is superior to state-of-the-art methods relying only on point correspondences.

\section{Relative Pose from Point and Line Matches} \label{sec:relpose}

\noindent
In this section, we study the problem of calibrated relative pose estimation between two images from 2D point correspondences (PC), line correspondences (LC), and the vanishing points (VP) stemming from parallel lines.  
Point correspondences can come from line junctions \cite{DBLP:conf/cvpr/ElqurshE11}, endpoints, or from an off-the-shelf feature detector and matcher, \eg, SuperPoint~\cite{superpoint} with SuperGlue~\cite{sarlin20superglue} or LoFTR~\cite{sun2021loftr}. 
Vanishing points are extracted from the detected line matches prior to the relative pose estimation procedure. 

\subsection{Theoretical Background}

\noindent
Here, we describe the theoretical concepts used in the paper.

\textbf{Projection matrix} $\Proj_i \in \RR^{3\times4}$ of the $i$-th camera is decomposed as $\Proj_i = \Intrinsic_i [\Rot_i \ \Tran_i]$, where $\Intrinsic_i \in \RR^{3\times3}$ is the intrinsic matrix, and $\Rot_i \in \text{SO}(3), \Tran_i \in \RR^3$ represent the rotation and translation, respectively.
In case of having calibrated cameras, $\Proj_i$ can be simplified to $\Proj_i = [\Rot_i \ \Tran_i]$.

\textbf{Relative pose} $(\Rot, \Tran)$ between two cameras $\Proj_1, \Proj_2$ is obtained as $\Rot = \Rot_2 \Rot_1^\text{T}$, $\Tran = \Tran_2 - \Rot_2 \Rot_1^\text{T} \Tran_1$.
The \textbf{epipolar geometry} \cite{DBLP:books/cu/HZ2004} relates the relative pose $(\Rot, \Tran)$ and a homogeneous 2D point correspondence $(\Point, \Point') \in \RR^3 \times \RR^3$ (PC). Let $\Worldpoint \in \RR^3$ be a point in space, $\Point$ be its projection into $\Proj_1$, and $\Point'$ be its projection into $\Proj_2$. 
Projections $\Point, \Point'$ are related by the epipolar constraint \cite{DBLP:books/cu/HZ2004} as ${\Point'}^\text{T} \Fund \Point = 0$,
where $\Fund$ is a fundamental matrix relating $\Proj_1$, $\Proj_2$. If the cameras are calibrated, the constraint is simplified to ${\Point'}^\text{T} \Ess \Point = 0$,
where $\Ess$ is the essential matrix relating cameras $\Proj_1$ and $\Proj_2$. The essential matrix $\Ess$ is decomposed as $[\Tran]_{\times} \Rot$. Then, the epipolar constraint is written as
\begin{equation}
    {\Point'}^\text{T} [\Tran]_{\times} \Rot \Point = 0 \label{eq:epipolar_pts}.
\end{equation}
Equation \eqref{eq:epipolar_pts} imposes one constraint on the relative pose $\Rot, \Tran$. Since the scale cannot be observed, the relative pose has five degrees of freedom, and it can be estimated from five point correspondences \cite{DBLP:conf/cvpr/Nister03}.

\textbf{Homography} relates planes projected into cameras $\Proj_1$ and $\Proj_2$. 
Let $\Plane$ be a 3D plane, and $\Worldpoint \in \Plane$ be a point on plane $\Plane$. 
Its projections $\Point$, $\Point'$ into $\Proj_1$, $\Proj_2$ are related by
%
    $\Point' \sim \Hom \Point$,
%
where $\Hom \in \RR^{3\times3}$ depends only on $\Proj_1$, $\Proj_2$, and $\Plane$.
Similarly, let $\Worldline_1 \subset \Plane$ be a line in its implicit form on plane $\Plane$. 
Its projections $\Line$, $\Line'$ into $\Proj_1$, $\Proj_2$ are related by
\begin{equation}
    \Line \sim \Hom^\text{T} \Line' \label{eq:hom_line}.
\end{equation}
We can estimate $\Hom$ from 4 coplanar line corrs.\ (LC) \cite{DBLP:books/cu/HZ2004}. 

\textbf{Vanishing point (VP)} is an intersection of 2D projections of parallel 3D lines. The homogeneous coordinates of vanishing point $\Vanishing_i$ in camera $j$ are $\Vanishing_i \sim \Intrinsic_j \Rot_j \Dir_i,$
where $\Dir_i \in \RR^3$ is the direction of the $i$-th line in 3D. If the camera is calibrated, then this formula is simplified as
\begin{equation}
    \Vanishing_i \sim \Rot_j \Dir_i \label{eq:VP_calibrated}.
\end{equation}
Let us have 2 calibrated cameras $\Proj_1 = [\Rot_1 \ \Tran_1]$, $\Proj_2 = [\Rot_2 \ \Tran_2]$. Vanishing points $\Vanishing$ in $\Proj_1$, $\Vanishing'$ in $\Proj_2$ are related by
\begin{equation}
    \Vanishing'  \sim \Rot_2 \Rot_1^\text{T} \Vanishing = \Rot \Vanishing \label{eq:van_sim},
\end{equation}
where $\Rot = \Rot_2 \Rot_1^\text{T} \in \text{SO}(3)$ is the relative rotation between $\Proj_1$, $\Proj_2$. Note that, if $\Vanishing$, $\Vanishing'$ are normalized, there must hold:
\begin{equation}
    \begin{split}
        \Vanishing' = \Rot \Vanishing \quad  \text{or} \quad \Vanishing' = -\Rot \Vanishing.
    \end{split} \label{eq:vanishing}
\end{equation}
A single vanishing point correspondence (VC) gives two constraints on rotation $\Rot$. 
Two VCs fully determine $\Rot$, and a third VC does not give additional constraints on the calibrated relative pose estimation.


\subsection{Possible Configurations}
\label{sec:possible_configs}
Considering the constraints described in the previous section, the number of distinct configurations of points, vanishing points, and lines orthogonal to them for estimating the relative pose is limited.
For their summary, see Fig.~\ref{fig:solvers_overview}. 
All other configurations can be traced back to these or do not provide additional constraints for the relative pose.

\begin{table}[]
    \centering
    \small
    \begin{tabular}{cccc}
        \toprule
        VPs & LC$\perp$VP & PC generic & PC coplanar \\
        \midrule
        0 & N/A & 5 & 0\\
        0 & N/A & 0 & 4\\
        1 & 0 & 3 & 0\\
        1 & 1 & 2 & 0\\
        2 & 0 & 2 & 0\\
        \bottomrule
    \end{tabular}
    \caption{\textbf{Overview of relevant configurations} using point correspondences (PC), vanishing points (VP), and line correspondences (LC) orthogonal to them. Each row corresponds to one family of configurations. PC and LC can be used interchangeably under the conditions of Section \ref{sec:possible_configs}.}
    \label{tab:configurations}
\end{table}

\begin{figure}
    \setlength{\tabcolsep}{4pt}
    \small
    \centering
    \begin{tabular}{c c c c c}
         \includegraphics[width=0.14\linewidth]{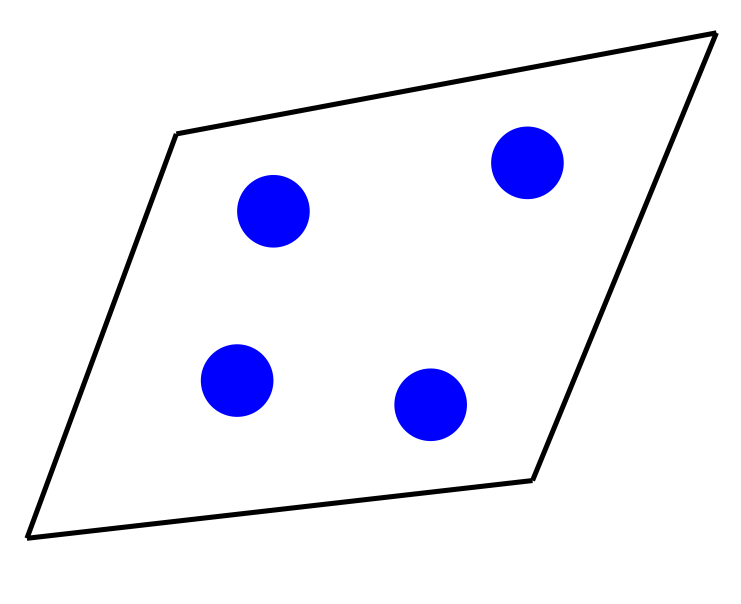} &
         \includegraphics[width=0.14\linewidth]{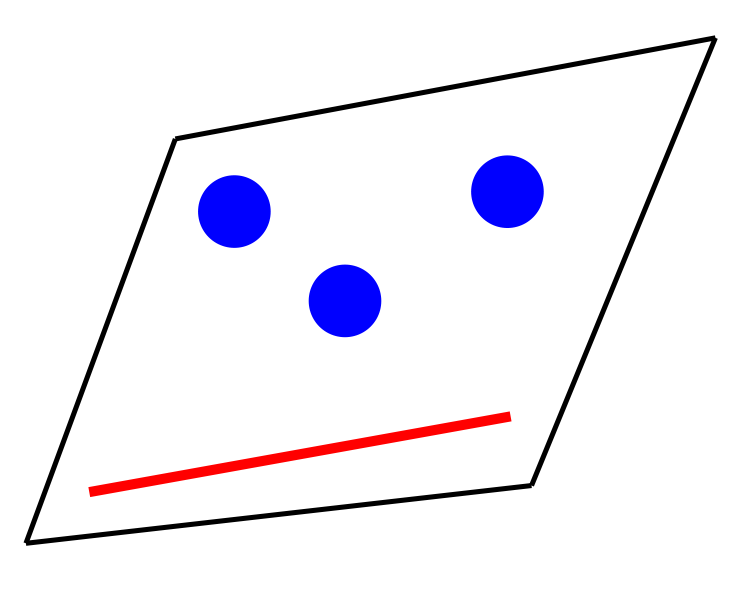} &
         \includegraphics[width=0.14\linewidth]{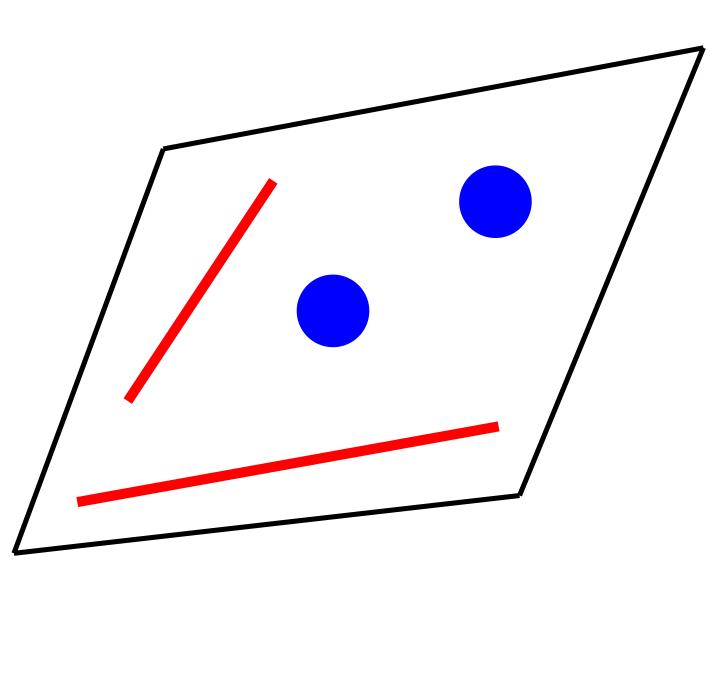} &
         \includegraphics[width=0.14\linewidth]{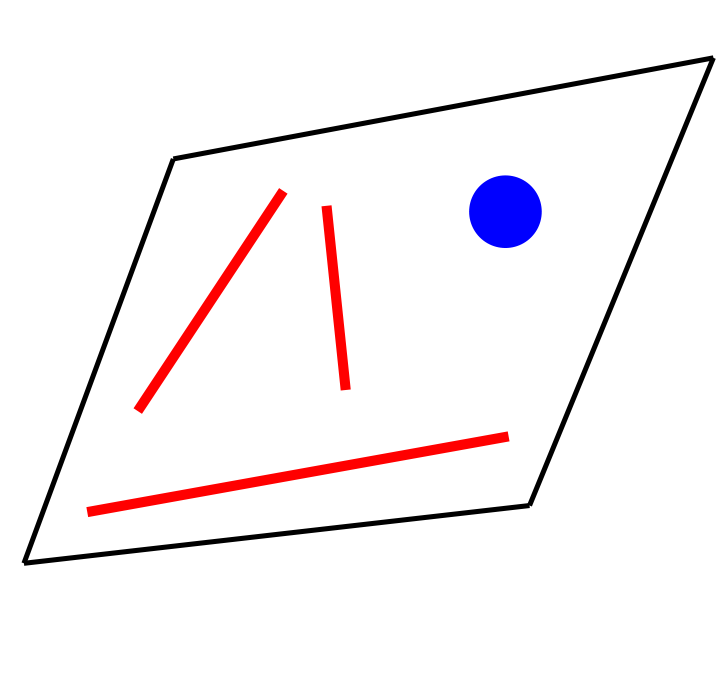} &
         \includegraphics[width=0.14\linewidth]{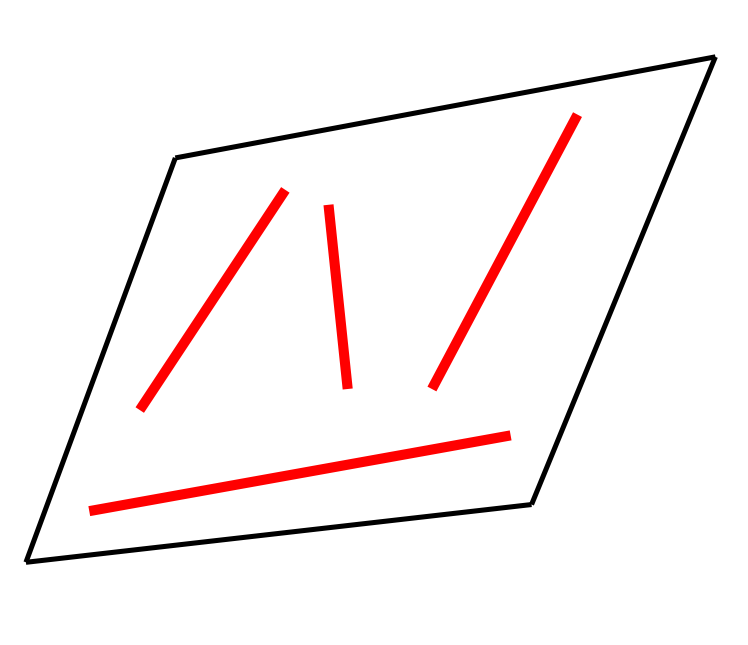} \\
         4-0-0 \cite{DBLP:books/cu/HZ2004} & 3-1-0 \cite{DBLP:books/cu/HZ2004} & 2-2-0 \cite{DBLP:books/cu/HZ2004} & 1-3-0 \cite{DBLP:books/cu/HZ2004} & 0-4-0 \cite{DBLP:books/cu/HZ2004} \\
    \end{tabular}\vspace{2mm}
         
    \begin{tabular}{c c c c}
        \includegraphics[width=0.2\linewidth]{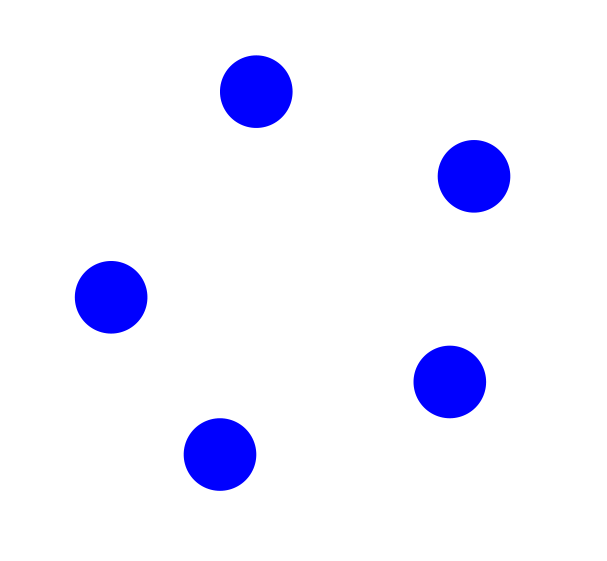} &
        \includegraphics[width=0.2\linewidth]{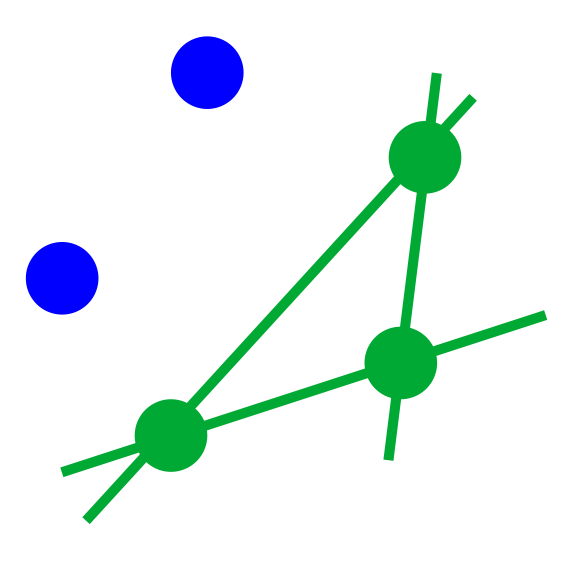} &
        \includegraphics[width=0.2\linewidth]{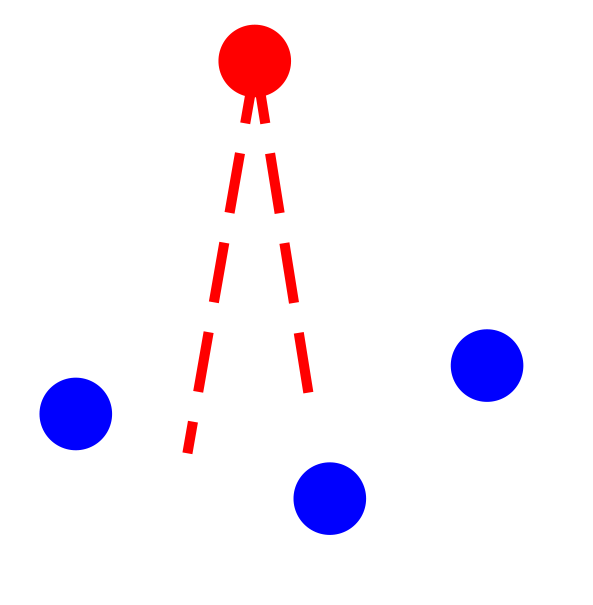} &
        \includegraphics[width=0.2\linewidth]{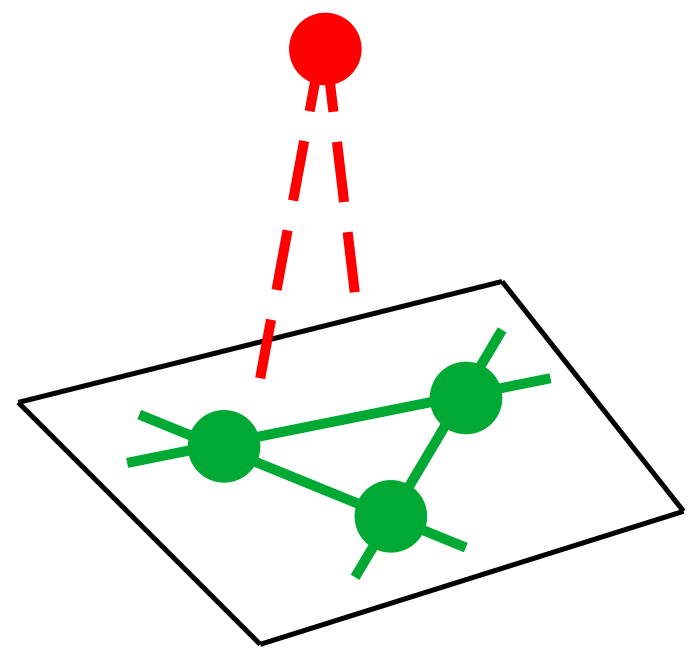}\\
        \\
        5-0-0 \cite{DBLP:conf/cvpr/Nister03} & 2-3-0 & 3-0-1 & 0-3-1 \\

        \includegraphics[width=0.2\linewidth]{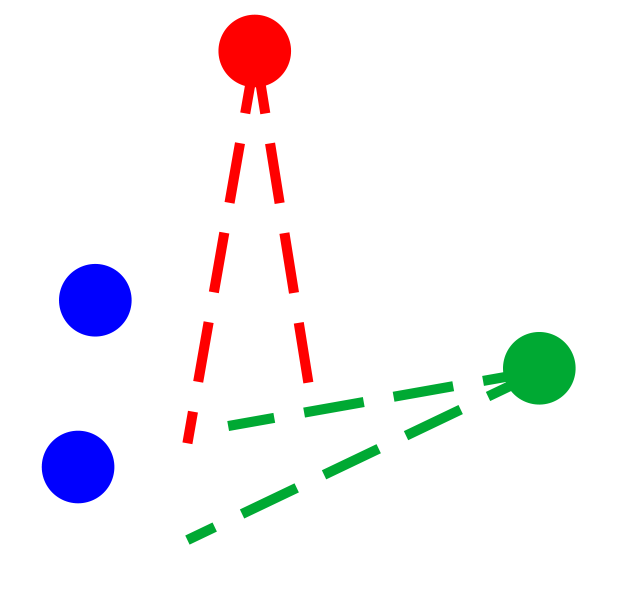} &
        \includegraphics[width=0.2\linewidth]{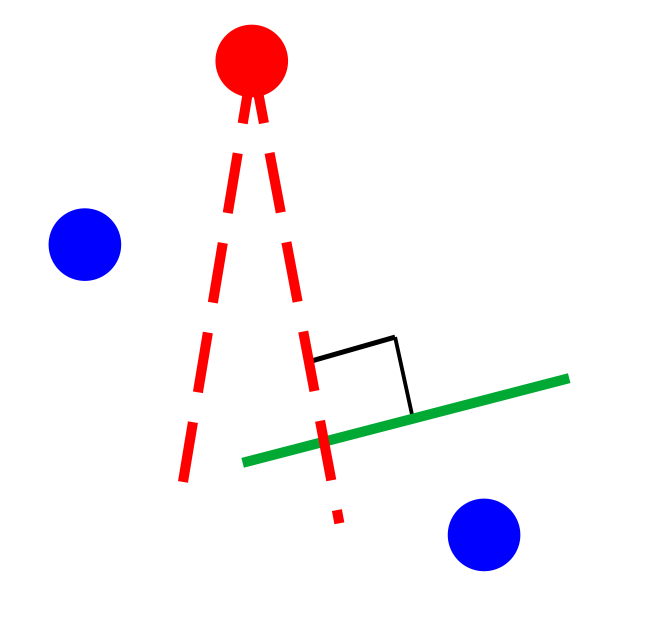} & \includegraphics[width=0.2\linewidth]{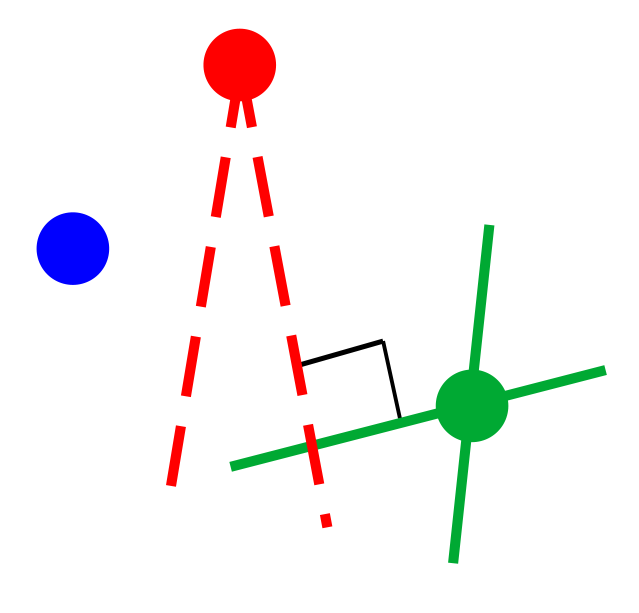} & 
        \includegraphics[width=0.2\linewidth]{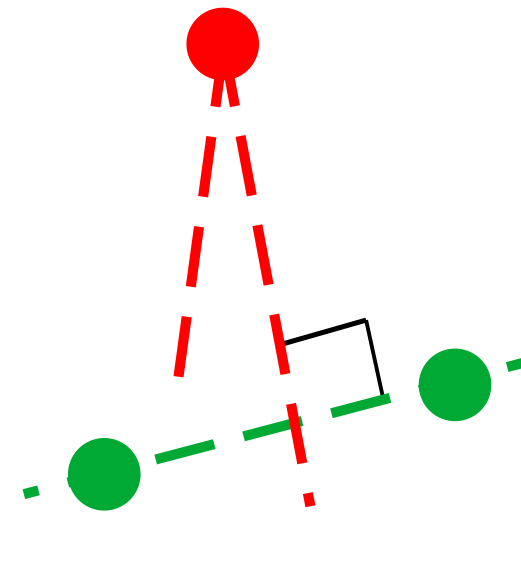}\\ 
        2-0-2 & 2-1-1$^\perp$ \cite{DBLP:conf/cvpr/ElqurshE11} & 1-2-1$^\perp$ & 2-0-1$^{\perp}$ \\
    \end{tabular}
    \caption{\textbf{Overview of the relevant solvers} showing configurations of points, lines, and vanishing points relevant to calibrated relative pose estimation. Configuration X-Y-Z: number of X points, Y lines, and Z vanishing points. }
    \label{fig:solvers_overview}
\end{figure}

~\\
\noindent
\textbf{Discussions on the completeness.}
We give here a high-level explanation that the list of configurations is complete. 
The full proof is provided in Sec.~A in the supp.\ mat.
The relative pose has 5 degrees of freedom (DoF) \cite{DBLP:conf/cvpr/Nister03}.
The possible configurations can only have $0$, $1$, or $2$ VPs.
While 1 VP fixes 2 DoF, and 2 VPs fix 3 DoF \cite{SalaunMM16}, a third VP does not provide additional constraints.
Moreover, one line orthogonal to a VP can create a second VP \cite{DBLP:conf/cvpr/ElqurshE11}; in the case of 2 VPs, a line orthogonal to one of them does not provide any new information.

A point correspondence fixes 1 DoF \cite{DBLP:books/cu/HZ2004}. Four coplanar points fix 5 DoF due to the homography constraint \cite{DBLP:books/cu/HZ2004}. 
Since $n<4$ points are always coplanar, their coplanarity does not add any new constraints. 

We use these facts to obtain the possible configurations of points, vanishing points, and lines orthogonal to them.
Overall, the configurations can be clustered into five distinct categories, summarized in Table~\ref{tab:configurations}.

\vspace{1mm}
\noindent
\textbf{Obtaining all configurations.}
To obtain more configurations, points can be replaced by lines with the following rules, where we adopt the solver notations of Fig. \ref{fig:solvers_overview}:
\begin{itemize}[noitemsep]
    \item 3 PCs can be replaced by \textit{3 coplanar lines}. Configuration 2-3-0 can thus be obtained from 5-0-0, and configuration 0-3-1 from 3-0-1.
    \item If we have 4 coplanar points, we can replace each of them with a line \cite{DBLP:books/cu/HZ2004}. Thus, the 4-0-0 configuration yields four additional ones: 3-1-0, 2-2-0, 1-3-0, 0-4-0.
    \item One PC can be replaced with an intersection of two lines.    
    We prove in Sec.~B of supp.\ mat.\ that using constraints implied by coplanar lines is equivalent to using their junctions as corresponding points. 
    \item 2-0-1$^{\perp}$, 2-1-1$^{\perp}$, and 1-2-1$^{\perp}$ belong to the same family.
\end{itemize}
In summary, the 5 categories in Table~\ref{tab:configurations} yield the 13 configurations of Fig.~\ref{fig:solvers_overview} that are relevant to the problem.

\subsubsection{Existing Solvers}\label{sec:existing}

Some of the previously listed problems have already been discussed in the literature and solved.
Such configurations are the 5 PC solver~\cite{DBLP:conf/cvpr/Nister03, DBLP:conf/icpr/LiH06, STEWENIUS2006284, DBLP:conf/bmvc/KukelovaBP08}, the 4 PC, 4 LC, and combined homography solvers \cite{DBLP:books/cu/HZ2004}, the 2-1-1$^{\perp}$ solver \cite{DBLP:conf/cvpr/ElqurshE11}, and the 2-0-2 solver \cite{SalaunMM16}.
Configuration 2-3-0 is solved by 5 PC solver~\cite{DBLP:conf/cvpr/Nister03} after replacing the junctions with points. Configurations 3-0-1 and 0-3-1 are solved by transforming to the 3 point upright relative pose problem \cite{DBLP:conf/eccv/FraundorferTP10, DBLP:journals/jmiv/KalantariHJG11, Gallier12, DBLP:conf/3dim/SweeneyFT14}.
We use these as off-the-shelf solvers in our experiments.
In the next sections, we review the minimal problems that have not been mentioned in the literature, and transform them to previously solved problems.

\subsubsection{Pose from 1VC and 3PCs (3-0-1)}\label{sec:1vp3pt}


We show here how to calculate the relative pose $\Rot$, $\Tran$ from one vanishing point correspondence and three point correspondences.
Suppose that we are given a vanishing point match $(\Vanishing, \Vanishing')$, and three point correspondences $(\Point_1, \Point_1')$, $(\Point_2, \Point_2')$, $(\Point_3, \Point_3')$, normalized by the camera intrinsics.

First, we are going to use the vanishing point to constrain rotation $\Rot$. The corresponding vanishing points are related by \eqref{eq:vanishing} which provides two systems of equations $\Vanishing'_1 = \Rot \Vanishing_1$ and $-\Vanishing'_1 = \Rot \Vanishing_1$.
If $\Rot \in \text{SO}(3)$ is a valid rotation matrix, it satisfies at least one of these systems. Each one is of form
\begin{equation}
    \mathbf{x}' = \Rot \mathbf{x}, \label{eq:13_general}
\end{equation}
where $\lVert \mathbf{x} \rVert = \lVert \mathbf{x}' \rVert = 1$. Estimating $\Rot$ and $\Tran$ with a constraint in form \eqref{eq:13_general} is similar to estimating the pose with known gravity \cite{DBLP:journals/jmiv/KalantariHJG11, DBLP:conf/eccv/FraundorferTP10, DBLP:conf/3dim/SweeneyFT14}.
We can resolve the sign ambiguity by checking the order of the lines, from which the VP was obtained. Based on \eqref{eq:13_general}, we can decompose $\Rot$ as:
\begin{equation}
    \Rot = {\Rot'}_{\mathbf{x}}^\text{T} \Rot_y \Rot_{\mathbf{x}}, \label{eq:13_rot_decomposition}
\end{equation}
where $\Rot_{\mathbf{x}}$ is a rotation that brings vector $\mathbf{x}$ to $y$-axis, ${\Rot}_{\mathbf{x}}'$ brings $\mathbf{x}'$ to $y$-axis, and $\Rot_y$ is rotation around the $y$-axis.

Let $\mathbf{b}_2 = [0 \ 1 \ 0]^\text{T}$ denote the $y$-axis direction. We find rotation $\Rot_{\mathbf{x}}$ using the Rodrigues formula as
\begin{equation*}
    \Rot_{\mathbf{x}} = \textbf{I} + \sin \alpha_x [\mathbf{a}_\mathbf{x}]_{\times} + (1-\cos \alpha_x) [\mathbf{a}_\mathbf{x}]_{\times}^2,
\end{equation*}
where $\alpha_x = \arccos \mathbf{x}^\text{T} \mathbf{b}_2$ is the angle between vector $\mathbf{x}$ and the $y$-axis, and $\mathbf{a}_\mathbf{x} = (\mathbf{x} \times \mathbf{b}_2) / \lVert \mathbf{x} \times \mathbf{b}_2 \rVert$ is the normalized cross product of $\mathbf{x}$ and the $y$-axis. We find rotation ${\Rot}_{\mathbf{x}}'$ in an analogous way.

Now, we are going to find rotation $\Rot_y$ and translation $\Tran$ from the point correspondences. From \eqref{eq:epipolar_pts}, there holds
\begin{equation}
    {\Point'_i}^\text{T} [\Tran]_{\times} {\Rot'}_{\mathbf{x}}^\text{T} \Rot_y(\varphi) \Rot_{\mathbf{x}} {\Point_i} = 0, \ i \in \{1,2,3\}, \label{eq:13_epipolar}
\end{equation}
where
\begin{equation*}
    \Rot_y(\varphi) = \begin{bmatrix}
        \cos \varphi & 0 & -\sin \varphi\\
        0 & 1 & 0\\
        \sin \varphi & 0 & \cos \varphi\\
    \end{bmatrix}.
\end{equation*}
The essential matrix $[\Tran]_{\times} {\Rot'}_{\mathbf{x}}^\text{T} \Rot_y(\varphi) \Rot_{\mathbf{x}}$ from \eqref{eq:13_epipolar} is equal to ${\Rot'}_{\mathbf{x}}^\text{T} [\Tran']_{\times} \Rot_y(\varphi) \Rot_{\mathbf{x}}$ where $\Tran' = {\Rot'}_{\mathbf{x}} \Tran$. We calculate $\Qoint_i = \Rot_{\mathbf{x}} \Point_i$, $\Qoint'_i = \Rot'_{\mathbf{x}} \Point'_i$, and convert system \eqref{eq:13_epipolar} to:
\begin{equation*}
    {\Qoint'_i}^\text{T} [\Tran']_{\times} \Rot_y(\varphi) {\Qoint'_i} = 0, \ i \in \{1,2,3\}.
\end{equation*}
This is the problem of estimating the relative pose with upright rotation from $3$ point correspondences. We solve this with the method from \cite{DBLP:conf/3dim/SweeneyFT14}, which reduces the problem to a $4$-degree univariate polynomial. 
Note that the straightforward approach would yield a $6$-degree polynomial. Finally, we compose $\Rot = {\Rot'}_{\mathbf{x}}^\text{T} \Rot_y(\varphi) \Rot_{\mathbf{x}}$, and $\Tran = {\Rot'}_{\mathbf{x}}^\text{T} \Tran'$. 
%
%
%
%

\subsubsection{Pose from 1VC and 3LCs (0-3-1)}\label{sec:031}

\noindent
Here, we discuss how to calculate the relative pose $\Rot, \Tran$ from one vanishing point correspondence and three correspondences of coplanar lines.
Suppose that we are given a corresponding vanishing point pair $(\Vanishing_1, \Vanishing_1')$ and three coplanar line correspondences $(\Line_1, \Line_1')$, $(\Line_2, \Line_2')$, $(\Line_3, \Line_3')$.
The line matches can be pairwise intersected to generate three corresponding point pairs as $\Point_i = \Line_i \times \Line_j$ and $\Point'_i = \Line'_i \times \Line'_j$, where $(i, j) \in \{(1,2), (1,3), (2,3)\}$.
We can use them together with the vanishing point correspondence to calculate the relative pose according to Sec.~\ref{sec:1vp3pt}.

\subsubsection{Pose from 1VC, 1PC, and 2LCs (1-2-1$^\perp$)}\label{sec:201_orthogonal}

If one of the lines $\Line_1$ is orthogonal to the direction of vanishing point $\Vanishing_1$, and the two lines intersect at $\Point_2 = \Line_1 \times \Line_2$, we can use line $\Line_1$ as the line $\Line$ together with $\Vanishing_1$, $\Point_1$, $\Point_2$, and find the relative pose according to the 2-1-1$^\perp$ solver \cite{DBLP:conf/cvpr/ElqurshE11}.

\subsubsection{Pose from 1VC, and 2PCs (2-0-1$^\perp$)}\label{sec:201_orthogonal_new}

If the line passing through the points $\Point_1$, $\Point_2$ is orthogonal to the direction of the vanishing point $\Vanishing_1$, we can use the line $\Line = \Point_1 \times \Point_2$ together with $\Vanishing_1$, $\Point_1$, $\Point_2$, and find the relative pose according to the 2-1-1$^\perp$ solver \cite{DBLP:conf/cvpr/ElqurshE11}.

\section{Joint VP Estimation and Matching} \label{sec:vp_matching}

In this work, we need to compute the vanishing points of pairs of images $I$ and $I'$, and to find the association between them. 
Instead of separately estimating vanishing points in both images and then matching them, we propose to jointly match and detect them at the same time. 

We first detect and match lines in two images using any existing line matcher~\cite{zhang2013lbd,Pautrat_Lin_2021_CVPR,syoon_2021_linetr,pautrat2023gluestick} and discard all the lines that are left unmatched. 
Given this association, we apply a multi-model fitting algorithm, \eg \cite{Barath_2019_ICCV}, to jointly detect the VPs in both images. 
To do so, we define a minimal solver, used inside \cite{Barath_2019_ICCV}, that gets $m = 2$ line-to-line correspondences as input, and returns the implied VP match.
Given line pairs $(\Line_1, \Line_2)$ in $I$, and $(\Line_1', \Line_2')$ in $I'$, this means that the corresponding VPs are calculated as $\Vanishing = \Line_1 \times \Line_2$ and $\Vanishing' = \Line_1' \times \Line_2'$.
For inlier counting in \cite{Barath_2019_ICCV}, we use the orthogonal distance in pixel, proposed in \cite{Tardif09}, to compute the distance from a VP to a line. 
A line match is considered inlier if its orthogonal distance is smaller than the inlier threshold in both images. 
Finally, we run the Levenberg-Marquardt numerical optimization~\cite{lma} on the inliers of each VP pair.

\section{Hybrid RANSAC on Points and Lines}

Now, we have a variety of solvers that utilize line, vanishing point and point (or junction) correspondences to estimate the relative pose.
However, it is unclear which solver works the best in practice -- there may not exist a best solver that works similarly well on all real-world scenarios.
The accuracy of a particular solver depends on the structure of the underlying scene and the configurations of geometric entities. 
For example, point features may be enough to recover relative poses for well-textured image pairs, while they fail completely in case of lack of distinctive texture. 
Thus, we aim to adaptively employ all solvers covered in this paper within a hybrid RANSAC framework~\cite{camposeco2018hybrid} to combine their advantages in a data-dependent manner. 

As proposed in \cite{camposeco2018hybrid}, at each iteration of RANSAC, we first sample a minimal solver with respect to a probability distribution computed from the prior distribution and the inlier ratios of the corresponding geometric entities of each solver. 
Then, we sample a minimal set corresponding to the selected solver and solve for the relative pose.
%
%
%
The termination criterion is adaptively determined for each solver similarly as in \cite{camposeco2018hybrid}, depending on the inlier threshold of the corresponding geometric entities and the predefined confidence parameter. 
As the correctness of the line correspondences cannot be verified from the estimated relative pose, we pre-set the line inlier ratio to be 0.6 for computing the probability distribution to sample the minimal solver for each iteration.
In our experiments, we set the prior probability to be uniform across all the solvers.
Finally, a Ceres-based non-linear optimization refines the estimated relative pose, minimizing the reprojection error on the point correspondences and the vanishing point error given the estimated rotation. 


\section{Experiments}

\subsection{Synthetic tests}


\noindent
\textbf{Numerical stability.} 
%
First, we generate a random rotation matrix $\Rot_\text{GT}$, and a translation vector $\Tran_\text{GT}$.
To generate a PC, we sample a point $\matr{X} \in \RR^{3}$ from a Gaussian distribution with mean $[0, \ 0, \ 5]^\text{T}$ and standard deviation $1$. We project $\matr{X}$ into the first camera as $\matr{p}$ and into the second one as $\matr{q}$.

To generate a LC in direction $\Dir$, we sample a 3D point $\matr{X}_A$ and a parameter $\lambda \in \RR$. We construct the second point as $\matr{X}_B = \matr{X}_A + \lambda \Dir$. 
We project these points into both images to get projections of the 3D line.
%
To generate vanishing point $\Vanishing_i$, we sample a direction $\Dir_i$. From $\Dir_i$, we generate two parallel 3D lines and project them into the images. Vanishing points $\Vanishing_i$ and $\Vanishing'_i$ are obtained as the intersections of the projected 2D lines. 
%
To generate a line orthogonal to a VP in direction $\Dir_i$, we sample a random direction $\Dir_0$, get direction $\Dir = \Dir_i \times \Dir_0$ orthogonal to $\Dir_i$, and sample a LC in direction $\Dir$.
%
To generate a tuple of $k$ coplanar lines, we generate $2k$ coplanar 3D points and use them as the endpoints. See the supplementary material for details. 


Let $\Rot_\text{est}$, $t_\text{est}$ be the rotation and translation estimated by a solver. 
We calculate the rotation error as the angle of the rotation represented as ${\Rot_\text{est}}^\text{T} \Rot_\text{gt}$, and the translation error as the angle between vectors $t_\text{est}$ and $t_\text{gt}$. 
Hence, we generated $n=100000$ random problem instances and ran the solvers on the noiseless samples.
Figure \ref{fig:stability_tests} shows histograms of pose errors on a representative subset of the solvers, all of which are stable -- there is no peak close to zero.

\noindent
\textbf{Tests with noise.} To evaluate the robustness of our solvers with respect to the input noise, we generate minimal problems similarly as in the previous section, and perturb the input with artificial noise.
Namely, we set the focal length $f=1000$, and add noise $\frac{\sigma}{f}$ to each calibrated endpoint.

To simulate the effect of junctions obtained from noisy lines, we generate two directions $\Dir_1$, $\Dir_2$, and four parameters $\lambda_1, \lambda_2, \lambda_3, \lambda_4 \in \RR$. Then, we get four endpoints $\matr{X}_1 = \matr{X} + \lambda_1 \Dir_1$, $\matr{X}_2 = \matr{X} - \lambda_2 \Dir_1$, $\matr{X}_3 = \matr{X} + \lambda_3 \Dir_2$, $\matr{X}_4 = \matr{X} - \lambda_4 \Dir_2$, project them into both cameras, add the noise to the projected endpoints, and take their junction.


Errors of a representative subset of the solvers with input noise are shown in Fig.~\ref{fig:noise_tests} without local optimization, Fig.~\ref{fig:lo_tests} with local optimization. The omitted solvers are extensions of the representative ones based on the same equations.

\noindent
\textbf{Orthogonality test.} Solvers 2-1-1$^{\perp}$, 1-2-1$^{\perp}$, and 2-0-1$^{\perp}$ assume that the angle between directions $\Dir_i$, $\Dir$ is $90^{\circ}$. We perturb the angle between $\Dir_i$, $\Dir$ and measure the error of those solvers. The result is shown in the supplementary.


\begin{figure}
    \centering
    \includegraphics[width=0.95\linewidth]{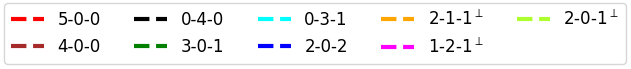}
       \includegraphics[width=0.99\linewidth,trim={0 3mm 0 0},clip]{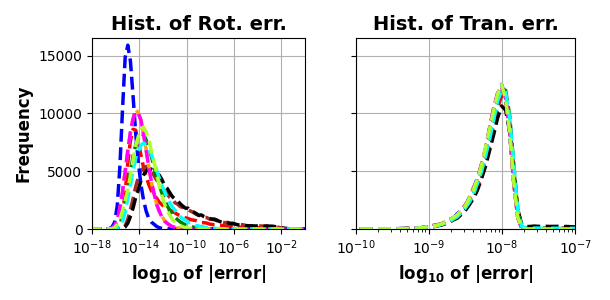}
    \caption{\textbf{Histograms of $\log_{10}$ rot.\ and trans.\ errors} in radians of minimal solvers computed from $100000$ noiseless samples.}
    \label{fig:stability_tests}
\end{figure}

\begin{figure}
    \centering
    \setlength\tabcolsep{0pt}
    \setlength\extrarowheight{-3pt}
    \renewcommand{\arraystretch}{0}
    \includegraphics[width=0.90\linewidth]{Figs/legend.png}
    \begin{tabular}{c c}
       \includegraphics[width=0.49\linewidth]{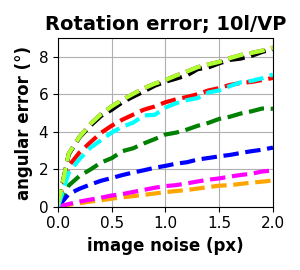}
       &
       \includegraphics[width=0.49\linewidth]{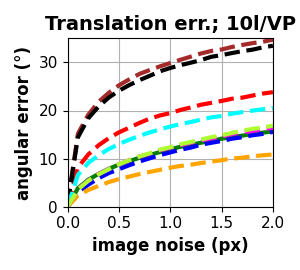}
       \\

       \includegraphics[width=0.49\linewidth]{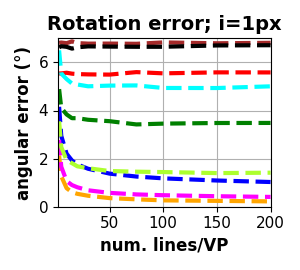}
       &
       \includegraphics[width=0.49\linewidth]{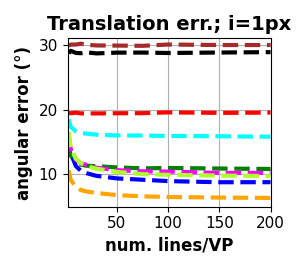}
       \\
    \end{tabular}
    
    \caption{
    \textbf{Average angular error in deg. of the proposed solvers} (see Fig.~\ref{fig:solvers_overview}) over 100000 runs, as a function of the image noise (\textit{top}), and the number of lines used for VP estimation (\textit{bottom}). 
    Image noise (i) and line number per VP (l/VP) are in the titles.}
    \label{fig:noise_tests}
\end{figure}

\begin{figure}
    \centering
    \setlength\tabcolsep{0pt}
    \setlength\extrarowheight{-3pt}
    \renewcommand{\arraystretch}{0}
    \includegraphics[width=0.92\linewidth]{Figs/legend.png}
    \begin{tabular}{c c}
       \includegraphics[width=0.49\linewidth]{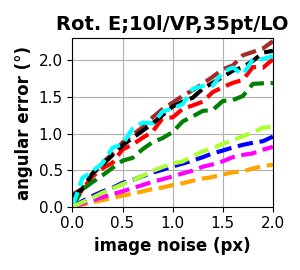}
       &
       \includegraphics[width=0.49\linewidth]{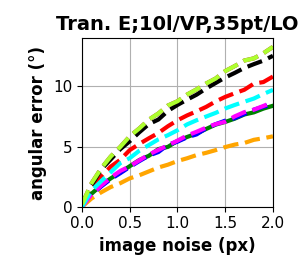}
       \\

       \includegraphics[width=0.49\linewidth]{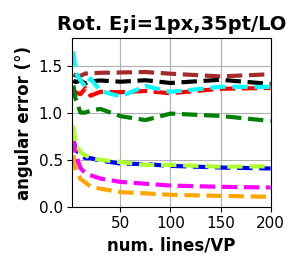}
       &
       \includegraphics[width=0.49\linewidth]{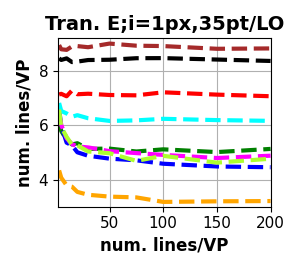}
       \\

       \includegraphics[width=0.49\linewidth]{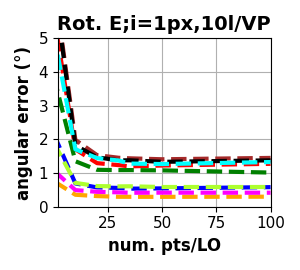}
       &
       \includegraphics[width=0.49\linewidth]{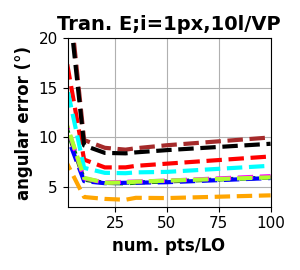}
       \\
    \end{tabular}
    
    \caption{
    \textbf{Angular error in deg. of the solvers}, averaged over 25000 runs, run with LO, as a function of (\textit{top}) the image noise (i), (\textit{middle}) the number of lines used for VP estimation (l/VP), and (\textit{bottom}) the number of points used inside LO (pt/LO).
    The fixed parameters for each test are reported in the titles.
    }
    \label{fig:lo_tests}
\end{figure}

%


\subsection{Real-world Experiments}


\vspace{1mm}
\noindent
\textbf{Datasets.}
We test our method on a variety of real-world datasets, both indoors and outdoors (see Table ~\ref{tab:datasets}). 

The 7Scenes dataset~\cite{7scenes} is a RGB-D dataset for visual localization, including 7 indoor scenes. We use the original GT poses provided with the images, and select pairs of images among all test sequences. 
Since each sequence is densely sampled, we take every 10th image $i$ and associate it with the image $i + 50$. 
ScanNet~\cite{dai2017scannet} is a large-scale RGB-D indoor dataset. It pictures some hard cases with low texture, where lines are expected to provide better contraints. 
We use the test set of 1500 images as in SuperGlue~\cite{sarlin20superglue}. 
The PhotoTourism dataset~\cite{phototourism} is a large-scale outdoor dataset of landmark pictures collected from the Internet, with GT poses from SfM. 
We reuse the validation pairs of the CVPR Image Matching Workshop 2020~\cite{Jin2020} with a total of 9900 pairs. 
ETH3D is an indoor-outdoor dataset~\cite{Schops_2017_eth3d}. 
We use the 13 scenes of the training set of the high resolution multi-view images, and sample all pairs of images with at least 500 GT keypoints in common.
The KITTI dataset \cite{Geiger2012CVPR, Geiger2013IJRR} is an outdoor dataset focused on the driving scenario. We use the 11 sequences of the training split of the Visual Odometry Challenge. 
For every sequence, we sample every 10th frame, and form pairs of consecutive images. 
This results in 2319 image pairs.
Finally, the LaMAR dataset \cite{sarlin2022lamar} is an indoor-outdoor dataset focused on augmented reality. 
We use the images of the validation split on Hololens in the CAB building. 
We use consecutive images to form pairs, resulting in 1423 pairs.

We use lines detected by DeepLSD~\cite{pautrat2023deeplsd} and matched with GlueStick~\cite{pautrat2023gluestick}.
While we were experimenting with a number of methods to obtain lines (\eg, LSD) and match (\eg, SOLD2) them, this combination leads to the best performance on all tested datasets. 
 The vanishing points are calculated from these lines by Prog-X~\cite{Barath_2019_ICCV}. 
We run LoFTR~\cite{sun2021loftr} to obtain point correspondences.
We generate junction correspondences from line segment pairs that actually intersect in both images.
This proved to be a good heuristic to obtain accurate junctions.
Also, we consider the line endpoints as additional point correspondences.

\begin{table}
    \centering
    \scriptsize
    \setlength{\tabcolsep}{4pt}
    \resizebox{1.0\columnwidth}{!}{\begin{tabular}{lcccccc}
        \toprule
         & 7Scenes~\cite{7scenes} & ScanNet~\cite{dai2017scannet} & PhotoT.~\cite{Jin2020} & ETH3D~\cite{Schops_2017_eth3d} & KITTI\cite{Geiger2012CVPR} & LaMAR\cite{sarlin2022lamar} \\
        \midrule
        \# images & 1610 & 1500 & 9900 & 1969 & 2319 & 1423  \\
        GT type & Kinect & Kin.\ + CAD & SfM & LiDAR & Laser & Laser \\
        Indoors & \checkmark & \checkmark & \xmark & \checkmark & \xmark & \checkmark \\ 
        Outdoors & \xmark & \xmark & \checkmark & \checkmark & \checkmark & \checkmark \\ 
        \bottomrule
    \end{tabular}}
    \caption{\textbf{Datasets overview.} We consider a variety of indoor/outdoor datasets with different GT modalities.}
    \label{tab:datasets}
\end{table}

\vspace{1mm} \noindent 
\textbf{7Scenes dataset.}
Table~\ref{tab:7scenes} presents the Area Under Curve (AUC) for the maximum rotation and translation errors, specifically $\max{(\epsilon_\textbf{R}, \epsilon_\textbf{t})}$, at error thresholds of $5^\circ$, $10^\circ$, and $20^\circ$. 
Additionally, it reports the median pose error in degrees and the average time in milliseconds on the 7Scenes dataset. 
The \textit{first} row presents the results of the proposed method when applied to LoFTR~\cite{sun2021loftr} point correspondences, effectively acting as MSAC with non-linear final optimization. 
The \textit{second} row, labeled 5PC + junc., extends this method to LoFTR correspondences and those derived from the endpoints of line matches and junctions.
The \textit{third} row, denoted as 5PC + 4PC, employs the 5PC essential matrix and 4PC homography solvers together, representing the scenario where all point-based solvers are used. 
The \textit{fourth} row integrates the aforementioned solvers but also includes the line endpoints and junctions.
The \textit{fifth} row introduces our proposed hybrid estimator, incorporating all point and line-based solvers. 
Finally, the last row also considers the line endpoints and junctions during estimation.

From the insights offered by Table~\ref{tab:7scenes}, the contribution of line endpoints and junctions appears non-essential on this dataset. However, the hybrid estimator proves superior, enhancing the AUC scores by 1-2 points.

\vspace{1mm} \noindent 
\textbf{ScanNet dataset.}
Table~\ref{tab:scannet} presents the AUC scores at error thresholds of $5^\circ$, $10^\circ$, and $20^\circ$, alongside the median pose error in degrees and the average runtime in milliseconds for the ScanNet dataset. 
Our observations regarding the effectiveness of junctions and endpoints in this dataset align with those from 7Scenes; their advantage in enhancing accuracy remains ambiguous.
While jointly employing the 5PC and 4PC solvers offers a noticeable accuracy improvement over solely using 5PC, with an average increase of approximately 1 AUC point, our proposed hybrid estimator realizes the most significant gains. 
This estimator, which integrates all line-based solvers, improves by 2-3 AUC points compared to point-only strategies. 
Even with these advancements, the computation remains marginally slower, ensuring real-time performance on this dataset.

\vspace{1mm} \noindent 
\textbf{PhotoTourism dataset.}
Table~\ref{tab:scannet} details the AUC scores at error thresholds of $5^\circ$, $10^\circ$, and $20^\circ$. Additionally, it reports the median pose error in degrees and the average runtime in milliseconds for the PhotoTourism dataset. Line junctions and endpoints appear to be counterproductive on this particular dataset, causing a notable decline in accuracy for point-based estimators.
Interestingly, our hybrid method manages to harness these elements, producing the most accurate results. It surpasses the baseline by a margin of 1-2 AUC points while maintaining real-time performance. 

\vspace{1mm} \noindent 
\textbf{ETH3D dataset.} 
Table~\ref{tab:eth3d} details the results on the ETH3D dataset. 
Our observations from this dataset mirror those from the PhotoTourism dataset. 
Specifically, point-based solutions experience a decrease in accuracy when the LoFTR correspondences are combined with those derived from line junctions and endpoints. 
However, our proposed hybrid methodology successfully harnesses this extra data, resulting in a noteworthy enhancement of 2-3 AUC points.

\vspace{1mm} \noindent 
\textbf{KITTI dataset.}
Table~\ref{tab:kitti} outlines the outcomes on the KITTI dataset. 
In this distinct setting -- characterized by a forward-moving camera -- line endpoints and junctions enhance the performance of all evaluated methods, often by a significant margin. 
The proposed hybrid method, integrated with endpoints and junctions showcases top-tier accuracy, and it is on par with 5PC + junc. 
Notably, in this scenario, the hybrid method is the fastest and it is second fastest when using additional correspondences from the line matches.

\vspace{1mm} \noindent 
\textbf{LaMAR dataset.}
Table~\ref{tab:lamar} presents the results on the LaMAR dataset. 
Within this dataset, the integration of line endpoints and junctions results in a marked improvement, enhancing accuracy by 3-6 AUC points. 
The proposed approach, which simultaneously utilizes LoFTR, endpoint, junction, and line matches, stands out. 
Compared to the baseline, this method manifests a substantial boost in performance, improving by 4-7 AUC points on average.

\begin{table}
    \small
    \centering
    \setlength{\tabcolsep}{4pt}
    \begin{tabular}{l | ccc | c c}
        \toprule
        \multirow{2}{*}{Solver} & \multicolumn{3}{c |}{Pose Accuracy $\uparrow$} & Med.\ $\downarrow$ & Time  \\
        & AUC@5$^\circ$ & @10$^\circ$ & @20$^\circ$ & err.\ ($^\circ$) & (ms) \\
        \midrule
        5PC & 16.3 & 36.6 & 57.5 & \underline{3.5} & \textbf{77.1} \\
        5PC + junc.\ & 16.2 & 36.8 & 57.9 & 3.6 & 96.7 \\
        5PC + 4PC & 16.1 & 36.8 & 57.9 & 3.6 & \underline{88.4} \\
        5PC + 4PC + junc.\ & 16.6 & 37.1 & 57.7 & \underline{3.5} & 98.5 \\
        \midrule
        Hybrid & \textbf{17.3} & \textbf{38.6} & \underline{59.1} & \textbf{3.4} & 206.0 \\
        Hybrid + junc.\ & \underline{16.8} & \underline{38.4} & \textbf{59.3} & \underline{3.5} & 214.0 \\
        \bottomrule
    \end{tabular}
    \caption{\textbf{Relative pose estimation on 7Scenes~\cite{7scenes}.} 
    We report the performance of the proposed method on LoFTR~\cite{sun2021loftr} point and DeepLSD + GlueStick~\cite{pautrat2023deeplsd,pautrat2023gluestick} line correspondences with the 5PC solver~\cite{DBLP:conf/cvpr/Nister03}, with the 5PC + 4PC solvers~\cite{DBLP:books/cu/HZ2004}, and with all line-based solvers (Hybrid) with line junctions and endpoints (+ junc).
    The best results are in \textbf{bold}, and the second bests are \underline{underlined}.}
    %
    \label{tab:7scenes}
\end{table}

\begin{table}
    \small
    \centering
    \setlength{\tabcolsep}{4pt}
    \begin{tabular}{l | ccc | c c}
        \toprule
        \multirow{2}{*}{Solver} & \multicolumn{3}{c |}{Pose Accuracy $\uparrow$} & Med.\ $\downarrow$ & Time  \\
        & AUC@5$^\circ$ & @10$^\circ$ & @20$^\circ$ & err.\ ($^\circ$) & (ms) \\
        \midrule
        5PC & 20.8 & 40.2 & 58.1 & 3.1 & \textbf{10.2} \\
        5PC + junc.\ & 20.9 & 39.8 & 58.3 & 3.2 & 32.6 \\
        5PC + 4PC & 21.7 & 41.0 & 58.7 & \underline{3.0} & 29.9 \\
        5PC + 4PC + junc.\ & 21.9 & 40.8 & 58.8 & \underline{3.0} & 22.6 \\
        \midrule
        Hybrid & \textbf{23.1} & \textbf{42.5} & \textbf{60.0} & \textbf{2.9} & \underline{22.6} \\
        Hybrid + junc.\ & \underline{22.3} & \underline{41.6} & \underline{59.4} & \underline{3.0} & 53.0 \\
        \bottomrule
    \end{tabular}
    \caption{\textbf{Relative pose estimation on ScanNet~\cite{dai2017scannet}.} We report the performance of the proposed method on LoFTR~\cite{sun2021loftr} point and DeepLSD + GlueStick~\cite{pautrat2023deeplsd,pautrat2023gluestick} line correspondences with the 5PC solver~\cite{DBLP:conf/cvpr/Nister03}, with the 5PC + 4PC solvers~\cite{DBLP:books/cu/HZ2004}, and with all line-based solvers (Hybrid) with line junctions and endpoints (+ junc).
    The best results are in \textbf{bold}, and the second bests are \underline{underlined}.
    }
    \label{tab:scannet}
\end{table}

\begin{table}
    \small
    \centering
    \setlength{\tabcolsep}{4pt}
    \begin{tabular}{l | ccc | c c}
        \toprule
        \multirow{2}{*}{Solver} & \multicolumn{3}{c |}{Pose Accuracy $\uparrow$} & Med.\ $\downarrow$ & Time  \\
        & AUC@5$^\circ$ & @10$^\circ$ & @20$^\circ$ & err.\ ($^\circ$) & (ms) \\
        \midrule
        5PC & 59.5 & \underline{74.6} & 85.4 & \textbf{0.9} & \underline{42.6} \\
        5PC + junc.\ & 54.0 & 70.2 & 82.7 & 1.1 & 71.7 \\
        5PC + 4PC & 58.8 & 73.9 & 85.0 & \underline{1.0} & \textbf{42.3} \\
        5PC + 4PC + junc.\ & 53.1 & 69.3 & 82.1 & 1.2 & 66.1 \\
        \midrule
        Hybrid & \textbf{61.3} & \textbf{75.9} & \underline{86.1} & \textbf{0.9} & 68.4 \\
        Hybrid + junc.\ & \underline{61.1} & \textbf{75.9} & \textbf{86.2} & \textbf{0.9} & 82.7 \\
        \bottomrule
    \end{tabular}
    \caption{\textbf{Relative pose estimation on PhotoTourism~\cite{Jin2020}.} 
    We report the performance of the proposed method on LoFTR~\cite{sun2021loftr} point and DeepLSD + GlueStick~\cite{pautrat2023deeplsd,pautrat2023gluestick} line matches with the 5PC solver~\cite{DBLP:conf/cvpr/Nister03}, with the 5PC + 4PC solvers~\cite{DBLP:books/cu/HZ2004}, and with all line-based solvers (Hybrid) with line junctions and endpoints (+ junc). The best results are in \textbf{bold}, and the second bests are \underline{underlined}.}
    \label{tab:phototourism}
\end{table}

\begin{table}
\small
    \centering
    \setlength{\tabcolsep}{4pt}
    \begin{tabular}{l | ccc | c c}
        \toprule
        \multirow{2}{*}{Solver} & \multicolumn{3}{c |}{Pose Accuracy $\uparrow$} & Med.\ $\downarrow$ & Time  \\
        & AUC@5$^\circ$ & @10$^\circ$ & @20$^\circ$ & err.\ ($^\circ$) & (ms) \\
        \midrule
        5PC & 72.1 & 81.0 & 86.3 & \underline{0.5} & \phantom{1}95.3 \\
        5PC + junc.\ & 70.2 & 80.3 & 86.3 & \underline{0.5} & 108.6 \\
        5PC + 4PC & 71.7 & 80.8 & 86.0 & \underline{0.5} & \phantom{1}84.0 \\
        5PC + 4PC + junc.\ & 70.2 & 80.3 & \underline{86.4} & \underline{0.5} & \phantom{1}98.6 \\
        \midrule
        Hybrid & \underline{75.0} & \underline{82.9} & \textbf{87.5} & \textbf{0.4} & \phantom{1}\textbf{56.0} \\
        Hybrid + junc.\ & \textbf{75.3} & \textbf{83.1} & \textbf{87.5} & \textbf{0.4} & \phantom{1}\underline{56.2} \\
        \bottomrule
    \end{tabular}
    \caption{\textbf{Relative pose estimation on ETH3D~\cite{Schops_2017_eth3d}.} 
    We report the performance of the proposed method on LoFTR~\cite{sun2021loftr} point and DeepLSD + GlueStick~\cite{pautrat2023deeplsd,pautrat2023gluestick} line correspondences with the 5PC solver~\cite{DBLP:conf/cvpr/Nister03}, with the 5PC + 4PC solvers~\cite{DBLP:books/cu/HZ2004}, and with all line-based solvers (Hybrid) with line junctions and endpoints (+ junc). The best results are in \textbf{bold}, and the second bests are \underline{underlined}.}
    \label{tab:eth3d}
\end{table}

\begin{table}
\small
    \centering
    \setlength{\tabcolsep}{4pt}
    \begin{tabular}{l | ccc | c c}
        \toprule
        \multirow{2}{*}{Solver} & \multicolumn{3}{c |}{Pose Accuracy $\uparrow$} & Med.\ $\downarrow$ & Time  \\
        & AUC@5$^\circ$ & @10$^\circ$ & @20$^\circ$ & err.\ ($^\circ$) & (ms) \\
        \midrule
        5PC & 61.8 & 70.6 & 75.8 & \textbf{0.7} & 277.8 \\
        5PC + junc.\ & \textbf{63.0} & \textbf{72.2} & \underline{77.7} & \textbf{0.7} & 321.4 \\
        5PC + 4PC & 61.1 & 70.4 & 75.7 & \textbf{0.7} & \textbf{198.0} \\
        5PC + 4PC + junc.\ & 61.6 & 71.0 & 77.0 & \textbf{0.7} & 238.6 \\
        \midrule
        Hybrid & 61.1 & 70.1 & 75.8 & \textbf{0.7} & \underline{229.7} \\
        Hybrid + junc.\ & \underline{62.4} & \textbf{72.2} & \textbf{77.9} & \textbf{0.7} & 250.5 \\
        \bottomrule
    \end{tabular}
    \caption{\textbf{Relative pose estimation on KITTI~\cite{Geiger2012CVPR, Geiger2013IJRR}.} 
    We report the performance of the proposed method on LoFTR~\cite{sun2021loftr} point and DeepLSD + GlueStick~\cite{pautrat2023deeplsd,pautrat2023gluestick} line correspondences with the 5PC solver~\cite{DBLP:conf/cvpr/Nister03}, with the 5PC + 4PC solvers~\cite{DBLP:books/cu/HZ2004}, and with all line-based solvers (Hybrid) with line junctions and endpoints (+ junc). The best results are in \textbf{bold}, and the second bests are \underline{underlined}.}
    \label{tab:kitti}
\end{table}

\begin{table}
\small
    \centering
    \setlength{\tabcolsep}{4pt}
    \begin{tabular}{l | ccc | c c}
        \toprule
        \multirow{2}{*}{Solver} & \multicolumn{3}{c |}{Pose Accuracy $\uparrow$} & Med.\ $\downarrow$ & Time  \\
        & AUC@5$^\circ$ & @10$^\circ$ & @20$^\circ$ & err.\ ($^\circ$) & (ms) \\
        \midrule
        5PC & 22.6 & 37.3 & 50.7 & 2.9 & \textbf{40.0} \\
        5PC + junc.\ & \underline{25.9} & \underline{41.6} & \underline{56.0} & \underline{2.4} & 64.3 \\
        5PC + 4PC & 22.2 & 37.2 & 51.1 & 3.0 & \underline{50.0} \\
        5PC + 4PC + junc.\ & 25.0 & 40.7 & 55.0 & 2.5 & 63.9 \\
        \midrule
        Hybrid & 24.7 & 39.8 & 53.0 & 2.7 & 81.2 \\
        Hybrid + junc.\ & \textbf{26.9} & \textbf{43.3} & \textbf{57.9} & \textbf{2.3} & 156.4 \\
        \bottomrule
    \end{tabular}
    \caption{\textbf{Relative pose estimation on LaMAR~\cite{sarlin2022lamar}.} 
    We report the performance of the proposed method on LoFTR~\cite{sun2021loftr} point and DeepLSD + GlueStick~\cite{pautrat2023deeplsd,pautrat2023gluestick} line correspondences with the 5PC solver~\cite{DBLP:conf/cvpr/Nister03}, with the 5PC + 4PC solvers~\cite{DBLP:books/cu/HZ2004}, and with all line-based solvers (Hybrid) with line junctions and endpoints (+ junc). The best results are in \textbf{bold}, and the second bests are \underline{underlined}.}
    \label{tab:lamar}
\end{table}

\vspace{1mm} \noindent 
\textbf{Vanishing Point Detection and Optimization.}
As described in Section~\ref{sec:vp_matching}, we simultaneously detect and match vanishing points in pairs of images, using the matched lines. The detection of VPs itself is done with Progressive-X~\cite{Barath_2019_ICCV}.
The proposed joint estimation runs for $2.95$ ms per pair on average on 7Scenes.
Running Prog-X independently on the images and then matching the VPs takes $3.67$ ms.
After detecting the VPs, we further refine them with a least square optimization using the Ceres solver~\cite{ceres}. For each vanishing point $\Vanishing$, we gather all inlier lines $\Line_i$ and re-fit the VP to these inliers, minimizing the sum of squared distances between the VP and the lines: $v_\text{refined} = \arg \min_\Vanishing \sum_{\Line_i} \delta(\Vanishing, \Line_i)^2$, where $\delta$ is the line-VP distance introduced in \cite{Tardif09}.

\begin{table}
    \centering
    \setlength{\tabcolsep}{4.8pt}
    \resizebox{0.7\columnwidth}{!}{\begin{tabular}{r | c c c c}
        \toprule
         &  Standard & VP joint & VP opt. & Both \\
        \midrule
        3-0-1 & 19.6 & 21.8 & 21.4 & \textbf{23.6} \\ 
        0-3-1 & \phantom{1}9.1 & 11.0 & 10.2 & \textbf{12.2} \\ 
        2-0-2 & \phantom{1}4.8 & \phantom{1}6.0 & \phantom{1}5.9 & \phantom{1}\textbf{7.1} \\ 
        2-1-1$^\perp$ & 20.4 & 21.8 & 21.5 & \textbf{23.1} \\ 
        1-2-1$^\perp$ & 17.7 & 19.3 & 18.2 & \textbf{20.1} \\ 
        \bottomrule
    \end{tabular}}
    \caption{\textbf{Ablation study of VP estimation on 7Scenes~\cite{7scenes}.} We report the AUC@10$^\circ$ score of representative solvers using VPs detected independently in each image and then matched (Standard), VPs detected jointly as proposed in Section~\ref{sec:vp_matching} (VP joint), VPs after numerical optimization (VP opt.), and when using both (Both).}
    \label{tab:vp_optimization}
\end{table}

The improvement from the optimization and the joint estimation is reported in Table~\ref{tab:vp_optimization} on the 7Scenes dataset when using the VP-based solvers independently.
Both the joint estimation and the optimization improve, and the best results are obtained when both are used to get accurate VPs.

%



\section{Conclusion}

In this paper, we have delved into exploiting 2D point and line correspondences to estimate the calibrated relative pose of two cameras. 
Our findings underscore that while leveraging line correspondences is not always straightforward, strategic incorporation of their endpoints, junctions, vanishing points, and line-based solvers can lead to a consistent improvement over traditional point-based methods. 
This nuanced approach improves across a diverse range of six datasets, from indoor to outdoor scenarios to applications in self-driving, mixed reality, and Structure-from-Motion. 
We believe our findings will serve as a helpful guide for those looking to use line correspondences in relative pose estimation. 
We will make the code publicly available.

%
%
%



{\footnotesize
\vspace{0.5em}
\noindent \textbf{Acknowledgments.}
We thank Marcel Geppert for helping to review this paper. Daniel Barath was supported by the ETH Postdoc Fellowship.
}

\newpage

\twocolumn[
    {\centering \Large Supplementary Material \\[1ex]}
    \vspace*{3ex}
]

\vspace{-10pt}
\appendix
\section{Complete List of Configurations} \label{sec:possible_configs}
In this section, we will provide the complete list of configurations that can be obtained using points, coplanar lines, vanishing points, and lines orthogonal to them. The section is separated into four subsections. In the first subsection, we give a complete list of configurations of points, vanishing points, and lines orthogonal to them. In the second subsection, we show how to extend these configurations with lines. In the last two subsections, we prove two propositions needed to obtain the list.


\subsection{Discussion on Completeness}

Here, we are going to give a complete list of configurations of points, vanishing points, and lines orthogonal to them that can be practically used for the estimation of relative pose between two cameras. To show this, we consider the following rules:
\begin{itemize}
    \item Calibrated relative pose has 5 degrees of freedom \cite{DBLP:conf/cvpr/Nister03}.
    \item The considered configurations have 0, 1, or 2 vanishing points (VPs) since a third vanishing point does not provide any new information (Sec.~\ref{sec:third_vanishing}).
    \item One vanishing point fixes 2 degrees of freedom. (Sec.~2.2.2 in the main paper)
    \item Two vanishing points fix 3 degrees of freedom. (Sec.~2.2.4 in the main paper)
    \item A line orthogonal to a vanishing point can create a second VP \cite{DBLP:conf/cvpr/ElqurshE11}, in the case of 2VPs, it does not add any new constraints (Sec.~\ref{sec:orthogonal_line}).
    \item A single point correspondence fixes a single degree of freedom as discussed in \cite{DBLP:books/cu/HZ2004}.
    \item Four coplanar points determine calibrated relative pose since via a homography \cite{DBLP:books/cu/HZ2004}, which can be decomposed to the relative pose \cite{20.500.11850/63639}.
    \item Since $n < 4$ points are always coplanar, their coplanarity does not add any new constraints.
\end{itemize}
Using these rules, we can obtain all possible configurations of points, vanishing points, and lines orthogonal to them that can be used for relative pose estimation. There are five such configurations that are summarized in Table~\ref{tab:configurations}.


\begin{table}[]
    \centering
    \small
    \begin{tabular}{ccccc}
        \toprule
        VPs & LC$\perp$VP & PC generic & PC coplanar & Code \\
        \midrule
        0 & N/A & 5 & 0 & 5-0-0\\
        0 & N/A & 0 & 4 & 4-0-0\\
        1 & 0 & 3 & 0 & 3-0-1\\
        1 & 1 & 2 & 0 & 2-1-1$^{\perp}$\\
        2 & 0 & 2 & 0 & 2-0-2\\
        \bottomrule
    \end{tabular}
    \caption{\textbf{Overview of relevant configurations} using point correspondences (PC), vanishing points (VP), and line correspondences (LC) orthogonal to them. Each row corresponds to one family of configurations. We give a code in format X-Y-Z, where X is the number of points, Y is the number of lines, and Z is the number of vanishing points.}
    \label{tab:configurations}
\end{table}

\subsection{Obtaining all configurations}
To obtain more configurations, we can replace points by lines according to the following rules:
\begin{itemize}
    \item Three point correspondences can be replaced by \textit{three coplanar lines}, that intersect in these points. Therefore, configuration 2-3-0 can be obtained from 5-0-0, and 0-3-1 from 3-0-1.
    \item If we have four coplanar points, we can replace each with a line in the same plane \cite{DBLP:books/cu/HZ2004} since the coplanar points and lines are transformed by the homography. 
    To estimate the homography from a minimal sample, the sum of the number of points and lines must be $4$. Therefore, we can obtain four new configurations from 4-0-0 as: 3-1-0, 2-2-0, 1-3-0, 0-4-0. 
    \item One point correspondence can be replaced with an intersection of two lines. We prove in Sec.~\ref{sec:junctions} that the constraints imposed by the coplanarity of the intersecting lines are equivalent to using the junction as a point correspondence, \ie, a line junction only gives us one independent constraint.
\end{itemize}

\noindent
Furthermore, the configuration 2-1-1$^{\perp}$ can be modified in the following way:
\begin{itemize}
    \item If one of the points is replaced with a line junction, and one of the lines building the line junction is orthogonal to the vanishing point, we obtain configuration 1-2-1$^{\perp}$.
    \item If the line passing through the points is orthogonal to the vanishing point, we obtain configuration 2-0-1$^{\perp}$.
    \item
    There are no other ways to use two points in order to obtain a line.
\end{itemize}
In both cases listed, it is possible to extract the orthogonal line, obtain the second VP from it, use both VPs to calculate rotation, and the points to calculate translation.

If we modify the configurations from Table~\ref{tab:configurations} with the rules from this section, we obtain 13 configurations shown in Figure~\ref{fig:solvers_overview}. Each of these configurations can be further modified by replacing any of the PCs with line junctions.

\begin{figure}
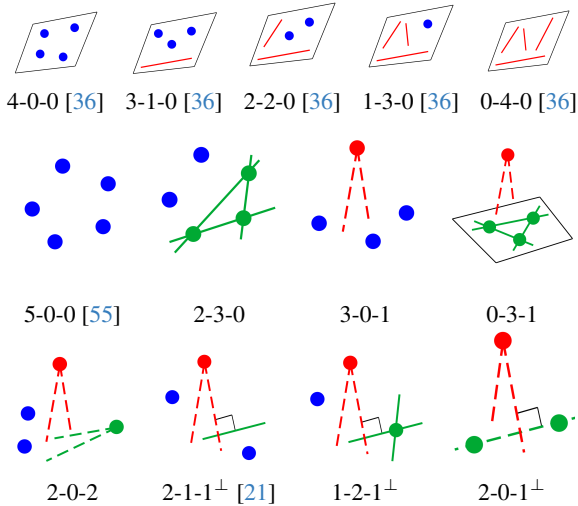

    \setlength{\tabcolsep}{4pt}
    \small
    \centering
    \begin{tabular}{c c c c c}
         \includegraphics[width=0.14\linewidth]{Figs/4-0-0.png} &
         \includegraphics[width=0.14\linewidth]{Figs/3-1-0.png} &
         \includegraphics[width=0.14\linewidth]{Figs/2-2-0.png} &
         \includegraphics[width=0.14\linewidth]{Figs/1-3-0.png} &
         \includegraphics[width=0.14\linewidth]{Figs/0-4-0.png} \\
         4-0-0 \cite{DBLP:books/cu/HZ2004} & 3-1-0 \cite{DBLP:books/cu/HZ2004} & 2-2-0 \cite{DBLP:books/cu/HZ2004} & 1-3-0 \cite{DBLP:books/cu/HZ2004} & 0-4-0 \cite{DBLP:books/cu/HZ2004} \\
    \end{tabular}\vspace{2mm}
         
    \begin{tabular}{c c c c}
        \includegraphics[width=0.2\linewidth]{Figs/5-0-0.png} &
        \includegraphics[width=0.2\linewidth]{Figs/2-3-0.png} &
        \includegraphics[width=0.2\linewidth]{Figs/1-0-3.png} &
        \includegraphics[width=0.2\linewidth]{Figs/0-3-1.png}\\
        \\
        5-0-0 \cite{DBLP:conf/cvpr/Nister03} & 2-3-0 & 3-0-1 & 0-3-1 \\

        \includegraphics[width=0.2\linewidth]{Figs/2-0-2.png} &
        \includegraphics[width=0.2\linewidth]{Figs/2-1-1_v2.png} & \includegraphics[width=0.2\linewidth]{Figs/2-0-1_v2.png} & 
        \includegraphics[width=0.2\linewidth]{Figs/2-0-1_p.png}\\ 
        2-0-2 & 2-1-1$^\perp$ \cite{DBLP:conf/cvpr/ElqurshE11} & 1-2-1$^\perp$ & 2-0-1$^{\perp}$ \\
    \end{tabular}
    \caption{\textbf{Overview of the relevant solvers} showing configurations of points, lines, and vanishing points relevant to calibrated relative pose estimation. Configuration X-Y-Z: number of X points, Y lines, and Z vanishing points. }
    \label{fig:solvers_overview}
\end{figure}

\subsection{Number of Vanishing Points in Calibrated Minimal Configurations}\label{sec:third_vanishing}

Here, we show that the congifigurations that can be practically used for estimating calibrated relative pose between two views contain 0, 1, or 2 vanishing points. 
%
%
Let us have three generic vanishing point correspondences $(\Vanishing_1, \Vanishing'_1)$, $(\Vanishing_2, \Vanishing'_2)$, $(\Vanishing_3, \Vanishing'_3)$. We want to find all relative poses $\Rot$, $\Tran$ that are consistent with the vanishing points.

We show in the main paper that for a generic configuration, the first two vanishing points are consistent with exactly $4$ rotation matrices $\Rot_a$, $\Rot_b$, $\Rot_c$, $\Rot_d$. Since the first two vanishing points already fix a finite set of rotations, the third vanishing point can only be used to constrain translation.
We want to find a set of all translations $\Tran \in \RR^3$ that satisfy the epipolar constraint as follows:
\begin{equation*}
    \Vanishing_3'^T [\Tran]_{\times} \Rot \Vanishing_3 = 0.
\end{equation*}
Because $(\Vanishing_3, \Vanishing'_3)$ is a vanishing point correspondence, there holds $\Vanishing_3' = \Rot \Vanishing_3$. Therefore, we can rewrite the epipolar constraint as
\begin{equation*}
    \Vanishing_3'^T [\Tran]_{\times} \Vanishing_3' = 0.
\end{equation*}
The left side is equal to ${\Vanishing'_3}^\text{T} (\Tran \times \Vanishing'_3)$, which is equal to zero for every $\Tran \in \RR^3$. Therefore, the third vanishing point does not impose any constraints on $\Tran$. Thus, practical configurations can only have 0, 1, or 2 vanishing points.
\\

\subsection{Use of a Line Orthogonal to a VP} \label{sec:orthogonal_line}
Here, we show that using a line orthogonal to a VP only makes sense if there is a single VP. Then, it can be used together with the VP to fix the rotation as shown in \cite{DBLP:conf/cvpr/ElqurshE11}.

It is clear that if there is no vanishing point, there cannot be any line orthogonal to a vanishing point.

Now, we will show that if there are two vanishing points, the line orthogonal to one of them does not fix any degrees of freedom. If there are two vanishing points, the rotation $\Rot$ is already fixed by these vanishing points. Therefore, the orthogonal line could only be used for fixing translation.

Let us have two vanishing point correspondences $(\Vanishing_1, \Vanishing'_1)$ and $(\Vanishing_2, \Vanishing'_2)$ that yield rotation matrix $\Rot$, and a line correspondence $(\Line, \Line')$ that is supposed to be orthogonal to the vanishing point $(\Vanishing_1, \Vanishing'_1)$.
Let $\textbf{L}$ denote the 3D line that can be obtained by backprojecting the 2D line $\Line$. The direction $\mathbf{d}$ of $\textbf{L}$ can be obtained as $\mathbf{d} = \Line \times \Vanishing_1$. Furthermore, the direction can be obtained as $\Rot^\text{T} (\Line' \times \Vanishing'_1)$. The solution exists if and only if $\Line \times \Vanishing_1 \sim \Rot^\text{T} (\Line' \times \Vanishing'_1)$.

If this equation does not hold, there is no solution. Therefore, let us suppose that this equation holds. Then, for every translation $\Tran \in \RR^3$, we can uniquely triangulate the 3D line $\textbf{L}$. The direction of this triangulated line $\textbf{L}$ has to be $\mathbf{d}$. Therefore, the orthogonality of line correspondence $(\Line, \Line')$ to a vanishing point correspondence $(\Vanishing_1, \Vanishing'_1)$ does not impose any constraints on translation $\Tran$, which concludes the proof that in the case of 2 VPs, a line orthogonal to a VP does not impose any new constraints on the relative pose.





\section{Relation of Coplanarity and Junctions} \label{sec:junctions}
In the main paper, we design solvers that leverage line junctions to obtain point correspondences from line correspondences. In this section, we give more details on this process, and we prove that the constraints implied by coplanar lines are equivalent to using line junctions. We use the notation from the main paper.

\noindent
\textbf{Line junctions.} If two lines $\Worldline_1$, $\Worldline_2$ in space intersect, they share a point $\Worldpoint \in \RR^3$. Let $\Line_1, \Line_2 \in \RR^3$ be homogeneous coordinates of the projections of $\Worldline_1$, $\Worldline_2$ into camera $\Proj$. Then, the projection $\Point$ of $\Worldpoint$ can be obtained as the intersection of $\Line_1, \Line_2$ as $\Point = \Line_1 \times \Line_2$.

Let us have two cameras $\Proj_1$, $\Proj_2$. Let $\Line_1, \Line_2$ be the projections of $\Worldline_1$, $\Worldline_2$ into $\Proj_1$, and $\Line'_1, \Line'_2$ the projections into $\Proj_2$. Then, the intersections $\Line_1 \times \Line_2$, and $\Line'_1 \times \Line'_2$ present a valid point correspondence between cameras $\Proj_1$, $\Proj_2$. According to the epipolar constraint, there holds:
\begin{equation} \label{eq:epipolar}
    (\Line'_1 \times \Line'_2)^\text{T} [\Tran]_{\times} \Rot (\Line_1 \times \Line_2),
\end{equation}
where $\Rot$, $\Tran$ is the relative pose between cameras $\Proj_1$, $\Proj_2$.
\\

\noindent
\textbf{Coplanar lines.} If two lines in space $\Worldline_1$, $\Worldline_2$ are coplanar, their projections $\Line_1$, $\Line_2$ into the first camera $\Proj_1$, and $\Line'_1$, $\Line'_2$ into the second camera $\Proj_2$ are related by the same homography matrix $\Hom$ as follows:
\begin{equation}
    \Line_1 \sim \Hom^\text{T} \Line'_1, \ \ \Line_2 \sim \Hom^\text{T} \Line'_2. \label{eq:same_homography}
\end{equation}
If the cameras $\Proj_1$, $\Proj_2$ are calibrated, the homography has the form $\Hom = \Rot - \Tran \Normline^\text{T}$, where $\Normline \in \RR^3$ is the normal of the plane defined by lines $\Worldline_1$, $\Worldline_2$. Then, \eqref{eq:same_homography} becomes:
\begin{equation}
    \Line_1 \sim (\Rot^\text{T} - \Normline \Tran^\text{T}) \Line'_1, \ \ \Line_2 \sim (\Rot^\text{T} - \Normline \Tran^\text{T}) \Line'_2. \label{eq:homography_calibrated}
\end{equation}

\subsection{Proof of Equivalence of Coplanar Lines and Junctions}
Now, we are going to show that the coplanarity constraint \eqref{eq:epipolar}, and the common homography constraint \eqref{eq:homography_calibrated} are equivalent, \ie, for generic $\Line_1$, $\Line_2$, $\Line'_1$, $\Line'_2$, the set of relative poses satisfying \eqref{eq:epipolar} is equal to the set of relative poses satisfying constraint \eqref{eq:homography_calibrated}.

\subsubsection{Coplanarity $\implies$ Junction} \label{sec:direction1}
First, we are going to show that if $\Rot$, $\Tran$ satisfies \eqref{eq:homography_calibrated} for generic projections $\Line_1$, $\Line_2$, $\Line'_1$, $\Line'_2$, then \eqref{eq:epipolar} holds.

\begin{proposition}
    Let $\Line_1$, $\Line_2$, $\Line'_1$, $\Line'_2$ be generic line projections. If $\Rot$, $\Tran$ satisfies \eqref{eq:homography_calibrated} for generic projections $\Line_1$, $\Line_2$, $\Line'_1$, $\Line'_2$, then \eqref{eq:epipolar} holds.
\end{proposition}

\begin{proof}
Let $\Rot$, $\Tran$ be a relative pose satisfying \eqref{eq:homography_calibrated}. Then, there is a normal $\Normline \in \RR^3$, for which \eqref{eq:homography_calibrated} holds. We get a cross product of the equations in \eqref{eq:homography_calibrated} to get:
\begin{equation*}
\begin{split}
    \Line_1 \times \Line_2 \sim (\Rot^\text{T} - \Normline \Tran^\text{T})^\text{-T} (\Line'_1 \times \Line'_2), \\
    (\Rot - \Tran \Normline^\text{T}) (\Line_1 \times \Line_2) \sim (\Line'_1 \times \Line'_2).
\end{split}
\end{equation*}
We multiply both sides of the equation by $[\Tran]_{\times}$ from left and use the fact that $\mathbf{x} \times \mathbf{x} = 0$ to get:
\begin{equation*}
    \begin{split}
        [\Tran]_{\times} (\Rot - \Tran \Normline^\text{T}) (\Line_1 \times \Line_2) \sim [\Tran]_{\times} (\Line'_1 \times \Line'_2),
        \\
        ([\Tran]_{\times} \Rot - [\Tran]_{\times} \Tran \Normline^\text{T}) (\Line_1 \times \Line_2) \sim [\Tran]_{\times} (\Line'_1 \times \Line'_2),
        \\
        [\Tran]_{\times} \Rot (\Line_1 \times \Line_2) \sim [\Tran]_{\times} (\Line'_1 \times \Line'_2).
    \end{split}
\end{equation*}
Now, we multiply both sides of the equation by $(\Line'_1 \times \Line'_2)^\text{T}$ and use the fact that $\mathbf{y}^\text{T} (\mathbf{x} \times \mathbf{y}) = 0$ to get:
\begin{equation*}
\begin{split}
    (\Line'_1 \times \Line'_2)^\text{T} [\Tran]_{\times} \Rot (\Line_1 \times \Line_2) \sim (\Line'_1 \times \Line'_2)^\text{T} [\Tran]_{\times} (\Line'_1 \times \Line'_2),
    \\
    (\Line'_1 \times \Line'_2)^\text{T} [\Tran]_{\times} \Rot (\Line_1 \times \Line_2) = 0.
\end{split}
\end{equation*}
This is exactly the epipolar constraint \eqref{eq:epipolar}.
\end{proof}

\subsubsection{Junction $\implies$ Coplanarity} \label{sec:direction2}
Now, we are going to show that if $\Rot$, $\Tran$ satisfies \eqref{eq:epipolar} for a generic configuration of $\Line_1$, $\Line_2$, $\Line'_1$, $\Line'_2$, then \eqref{eq:homography_calibrated} holds.

\begin{proposition}
    Let $\Line_1$, $\Line_2$, $\Line'_1$, $\Line'_2$ be a generic configuration of line projections. If $\Rot$, $\Tran$ satisfies \eqref{eq:epipolar} for a generic configuration of $\Line_1$, $\Line_2$, $\Line'_1$, $\Line'_2$, then \eqref{eq:homography_calibrated} holds.
\end{proposition}

\begin{proof}
Let $\Rot$, $\Tran$ be a relative pose satisfying \eqref{eq:epipolar}. Then, there is a point $\Worldpoint \in \RR^3$ that is projected by $\Proj_1$ to the intersection $\Line_1 \times \Line_2$, and by $\Proj_2$ to the intersection $\Line'_1 \times \Line'_2$ \cite{DBLP:books/cu/HZ2004}. Let $\Plane_1$ be the plane obtained by backprojecting $\Line_1$. Then, $\Worldpoint$ lies on $\Plane_1$. Similarly, $\Worldpoint$ lies on $\Plane'_1$ obtained by backprojecting $\Line'_1$. Therefore, $\Worldpoint$ lies on the intersection $\Plane_1 \cap \Plane'_1$ of $\Plane_1$, $\Plane'_1$. Since we assume a generic configuration, the intersection $\Plane_1 \cap \Plane'_1$ is a unique line \cite{DBLP:books/cu/HZ2004}, which equals to $\Worldline_1$. Therefore, $\Worldpoint$ lies on $\Worldline_1$. We use the same argument to show that $\Worldpoint$ lies on $\Worldline_2$. Therefore, lines $\Worldline_1$, $\Worldline_2$ are coplanar, and their projections are related by homography \cite{DBLP:books/cu/HZ2004}.
\end{proof}

For the sake of completeness, we also show the implication for the degenerate configurations of lines. In the degenerate configuration, the junctions $\Line_1 \times \Line_2$, $\Line'_1 \times \Line'_2$ are the vanishing points. Then, the lines $\Worldline_1$, $\Worldline_2$ are parallel, and therefore, they are coplanar.

\subsubsection{Junction $\iff$ Coplanarity}\label{sec:junction_equivalence}
Here, we combine the results from Section \ref{sec:direction1}, and from Section \ref{sec:direction2} to conclude the proof.

\begin{proposition}
Let $\Line_1$, $\Line_2$, $\Line'_1$, $\Line'_2$ be a set of generic line projections. Then, the constraint \eqref{eq:homography_calibrated} imposed on the relative pose $\Rot$, $\Tran$ by the coplanarity of the lines is equivalent to the epipolar constraint \eqref{eq:epipolar} imposed by the junction pair $(\Line_1 \times \Line_2, \Line'_1 \times \Line'_2)$.
\end{proposition}

\begin{proof}
We have shown in Section \ref{sec:direction1} that if the relative pose $\Rot$, $\Tran$ satisfies constraint \eqref{eq:homography_calibrated}, then it satisfies constraint \eqref{eq:epipolar}. In Section \ref{sec:direction2}, we have shown that if the relative pose $\Rot$, $\Tran$ satisfies \eqref{eq:epipolar}, then it satisfies \eqref{eq:homography_calibrated}. Therefore, the constraints imposed on the relative pose $\Rot$, $\Tran$ by \eqref{eq:epipolar}, and by \eqref{eq:homography_calibrated} are equivalent.
\end{proof}

\subsection{Alternative Solver Formulation Based on Coplanar Lines}
Here, we give an alternative way to eliminate $\Normline$ from system \eqref{eq:homography_calibrated}, and propose an alternative formulation of solvers 3-0-1, and 2-0-2 that use this constraint. We compare these alternative solvers with the proposed ones.

Constraint \eqref{eq:homography_calibrated} gives $4$ constraints that are linear in the elements of $\Normline$. Therefore, we can write the constraints as:
\begin{equation*}
    \matr{A}(\Rot, \Tran) \begin{bmatrix}
        \Normline \\ 1
    \end{bmatrix} = 0,
\end{equation*}
where $\matr{A}(\Rot, \Tran) \in \RR^{4,4}$ is matrix whose elements are functions of the relative pose $\Rot$, $\Tran$. We eliminate $\Normline$ to get:
\begin{equation}
    \det ( \matr{A}(\Rot, \Tran) ) = 0,\label{eq:determinant}
\end{equation}
which gives one constraint on the relative pose $\Rot$, $\Tran$. We tried to use the constraints in this form (instead of epipolar constraint on junctions) to solve the 3-0-1 and 2-0-2 problems from the main paper. Namely, in the 3-0-1 case, we have a rotation in the form ${\Rot'}_{\mathbf{x}}^\text{T} \Rot_y(\varphi) \Rot_{\mathbf{x}}$, and we need 3 constraints in the form \eqref{eq:determinant} to solve for $\varphi$, $\Tran$. In the 2-0-2 case, the rotation $\Rot$ is fixed by the vanishing points, and we need 2 constraints \eqref{eq:determinant} to solve for $\Tran$. In both cases, we employ the automatic minimal solver generator \cite{DBLP:journals/corr/abs-2004-11765} to get the minimal solvers. Table \ref{tab:time_comparison} shows the time comparison of this approach with the proposed junction-based solvers.  Figure \ref{fig:coplanar_stability_tests} shows the numerical stability comparison. We can see that this alternative approach is inferior compared to the proposed ones in terms of both stability and time.

\begin{table}[]
    \centering
    \begin{tabular}{c|c|c}
        & Junctions \eqref{eq:epipolar} & Alternative \eqref{eq:determinant} \\
        \hline
        3-0-1 &\textbf{ 8.67859 $\mu s$} & 1633.69 $\mu s$ \\
        2-0-2 & \textbf{0.13719 $\mu s$} & 24.6418 $\mu s$
    \end{tabular}
    \caption{Average time in $\mu s$ of minimal solvers using the junctions \eqref{eq:epipolar} vs. the alternative approach \eqref{eq:determinant}.}
    \label{tab:time_comparison}
\end{table}

\begin{figure}
    \centering
    \includegraphics[width=0.7\linewidth]{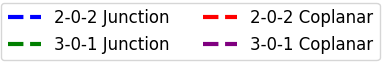}
       \includegraphics[width=0.95\linewidth,trim={0 3mm 0 0},clip]{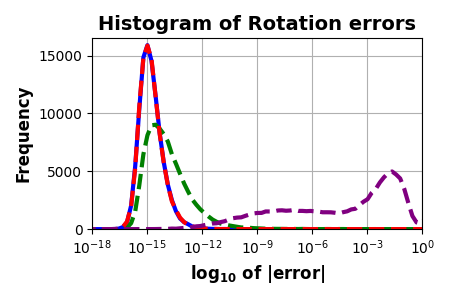}
       \includegraphics[width=0.95\linewidth,trim={0 3mm 0 0},clip]{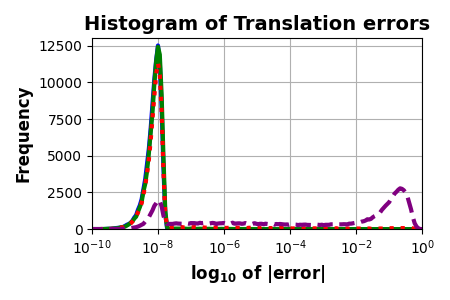}
    \caption{\textbf{Histogram of $\log_{10}$ pose errors} in radians of minimal solvers 3-0-1, and 2-0-2 computed from $100000$ noiseless samples. We compare the Junctions \eqref{eq:epipolar} and Coplanarity \eqref{eq:determinant} formulations. }
    \label{fig:coplanar_stability_tests}
\end{figure}


\section{Details on Experiments}
\subsection{Sampling of Synthetic Scenes}
Here, we describe in detail, how we sample the synthetic scenes, which we then use in the numerical stability and noise robustness tests.

\textbf{Relative pose and point correspondences (PC).} First, we generate a random rotation matrix $\Rot_\text{GT}$, and a camera center $\textbf{C}_\text{GT}$ from a Gaussian distribution with zero mean and unit standard deviation. 
We calculate translation $\Tran_\text{GT} = -\Rot_\text{GT} \textbf{C}_\text{GT}$.
To generate a PC, we sample a point $\matr{X} \in \RR^{3}$ from a Gaussian distribution with mean $[0 \ 0 \ 5]^\text{T}$, and standard deviation $1$. Then, we project the point to the first camera as $\matr{p} = \matr{X}$ and into the second one as $\matr{q} = \Rot_\text{GT} \matr{X}_{jA} + \Tran_\text{GT}$.

\textbf{Line correspondences (LC).} To generate a LC in direction $\Dir$, we sample a 3D point $\matr{X}_A$ and a parameter $\lambda \in \RR$. We construct the second point as $\matr{X}_B = \matr{X}_A + \lambda \Dir$. 
Then, we get the projections $\matr{p}_A$, $\matr{p}_B$ of both endpoints in the first camera, and $\matr{q}_A$, $\matr{q}_B$ in the second one. Then, we obtain the homogeneous coordinates of the line projections as $\Line = \matr{p}_{A} \times \matr{p}_{B}$, $\Line' = \matr{q}_{A} \times \matr{q}_{B}$.

\textbf{Vanishing points} To generate vanishing point correspondence $(\Vanishing_i, \Vanishing'_i)$, we first sample a direction $\Dir_i$. In the \textbf{numerical stability} tests, we generate two line correspondences $(\Line_1, \Line'_1)$, $(\Line_2, \Line'_2)$ in direction $\Dir_i$ according to the previous paragraph, and obtain vanishing points $\Vanishing_i$ and $\Vanishing'_i$ as the intersections of the projected 2D lines: $\Vanishing_i = \Line_1 \times \Line_2$, $\Vanishing'_i = \Line'_1 \times \Line'_2$. In the \textbf{noise robustness tests}, we generate $l \geq 3$ line correspondences $(\Line_j, \Line'_j), j \in \{1,...,l\}$ in direction $\Dir_i$, and add the noise $\frac{\sigma}{f}$ to the endpoints $\matr{p}_{A} \times \matr{p}_{B}$, $\matr{m} = \matr{q}_{A}$. Then, we construct matrix $\matr{A} \in \RR^{l,3}$ with rows $\Line_j, j \in \{1,...,l\}$, and find $\Vanishing_i$ as the least-squares solution to system $\matr{A} \Vanishing_i = 0$. We find vanishing point $\Vanishing'_i$ analogously from lines $\Line'_j, j \in \{1,...,l\}$.

\textbf{Line orthogonal to VP.} To generate a line orthogonal to a VP in direction $\Dir_i$, we sample a random direction $\Dir_0$, get direction $\Dir = \Dir_i \times \Dir_0$ orthogonal to $\Dir_i$, and sample a LC in direction $\Dir$.

\textbf{Tuple of coplanar lines.} To generate a tuple of $k$ coplanar lines, we first generate three points $\matr{X}_1, \matr{X}_2, \matr{X}_3 \in \RR^3$, and for every line, we sample 4 parameters $\lambda_1, ..., \lambda_4 \in \RR$ from a normalized Gaussian distribution. Then, we get the endpoints of the line as $\matr{X}_A = \matr{X}_1 + \lambda_1 \matr{X}_2 + \lambda_2 \matr{X}_3$, $\matr{X}_B = \matr{X}_1 + \lambda_3 \matr{X}_2 + \lambda_4 \matr{X}_3$. Then, we project the endpoints into both cameras and join them to get a line correspondence.

For every solver, we generate the entities that are needed to compute the pose.

\subsection{Additional Synthetic Tests}
Here, we give additional synthetic tests to evaluate the solvers 2-1-1$^{\perp}$, 1-2-1$^{\perp}$, and 2-0-1$^{\perp}$ from the main paper. These solvers assume that the angle, between the directions $\Dir_i$ of the vanishing point and $\Dir$ of the line, is $90^{\circ}$. We perturb this angle and measure the error of these solvers. The result are shown in Figure~\ref{fig:noise_orthogonality_test}. We consider both the case without local optimization, and after the local optimization is applied. The plot shows that the solvers 2-1-1$^{\perp}$, 1-2-1$^{\perp}$, and 2-0-1$^{\perp}$ are very robust to the deviation from the orthogonal direction, especially if they are combined with the local optimization. Even with deviation $10^{\circ}$, the average rotation error of the solvers does not exceed $1.1^{\circ}$, and the average translation error reaches about $4^{\circ}$ for 2-1-1$^{\perp}$, and $6^{\circ}$ for 1-2-1$^{\perp}$ and 2-0-1$^{\perp}$.

%

\begin{figure}
    \centering
    \setlength\tabcolsep{0pt}
    \setlength\extrarowheight{-3pt}
    \renewcommand{\arraystretch}{0}
    \includegraphics[width=0.48\linewidth]{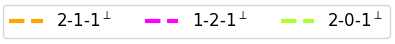}
    \begin{tabular}{c c}
       \includegraphics[width=0.43\linewidth]{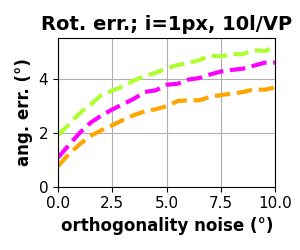}
       &
       \includegraphics[width=0.43\linewidth]{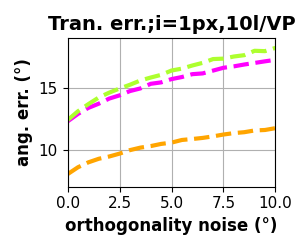}
       \\
       \includegraphics[width=0.45\linewidth]{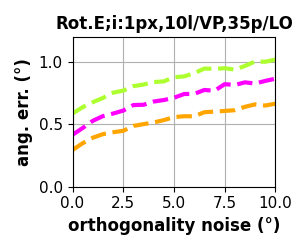}
       &
       \includegraphics[width=0.45\linewidth]{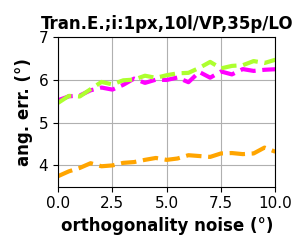}
       \\
    \end{tabular}
    
    \caption{
    \textbf{Angular error in deg. of solvers 2-1-1$^{\perp}$ and 1-2-1$^{\perp}$ \textit{top}: without local optimization, and \textit{bottom}: with local optimization.} Averaged over 25000 runs, as a function of deviation from the orthogonal direction. 
    Image noise (i), line number per VP (l/VP), and the number of points used in LO (p/LO) are in the titles.
    }
    \label{fig:noise_orthogonality_test}
\end{figure}

{\small
\bibliographystyle{ieeenat_fullname}
\bibliography{arxiv}

\begin{thebibliography}{87}
\providecommand{\natexlab}[1]{#1}
\providecommand{\url}[1]{\texttt{#1}}
\expandafter\ifx\csname urlstyle\endcsname\relax
  \providecommand{\doi}[1]{doi: #1}\else
  \providecommand{\doi}{doi: \begingroup \urlstyle{rm}\Url}\fi

\bibitem[Abdellali et~al.(2021)Abdellali, Frohlich, Vilagos, and Kato]{abdellali2021l2d2}
Hichem Abdellali, Robert Frohlich, Viktor Vilagos, and Zoltan Kato.
\newblock L2d2: Learnable line detector and descriptor.
\newblock In \emph{3DV}, 2021.

\bibitem[Agarwal and Mierle()]{ceres}
Sameer Agarwal and Keir Mierle.
\newblock Ceres solver.
\newblock \url{http://ceres-solver.org}.

\bibitem[Agarwal et~al.(2011)Agarwal, Furukawa, Snavely, Simon, Curless, Seitz, and Szeliski]{DBLP:journals/cacm/AgarwalFSSCSS11}
Sameer Agarwal, Yasutaka Furukawa, Noah Snavely, Ian Simon, Brian Curless, Steven~M. Seitz, and Richard Szeliski.
\newblock Building rome in a day.
\newblock \emph{Commun. {ACM}}, 54\penalty0 (10), 2011.

\bibitem[Barath and Hajder(2018)]{DBLP:journals/tip/BarathH18}
Daniel Barath and Levente Hajder.
\newblock Efficient recovery of essential matrix from two affine correspondences.
\newblock \emph{IEEE TIP}, 27\penalty0 (11), 2018.

\bibitem[Barath and Kukelova(2022)]{barath2022relative}
Daniel Barath and Zuzana Kukelova.
\newblock Relative pose from sift features.
\newblock \emph{ECCV}, 2022.

\bibitem[Barath and Matas(2018)]{barath2018graph}
Daniel Barath and Jiri Matas.
\newblock Graph-cut ransac.
\newblock In \emph{Proceedings of the IEEE conference on computer vision and pattern recognition}, pages 6733--6741, 2018.

\bibitem[Barath and Matas(2019)]{Barath_2019_ICCV}
Daniel Barath and Jiri Matas.
\newblock Progressive-{X}: Efficient, anytime, multi-model fitting algorithm.
\newblock In \emph{ICCV}, 2019.

\bibitem[Barath et~al.(2020)Barath, Noskova, Ivashechkin, and Matas]{DBLP:conf/cvpr/BarathNIM20}
Daniel Barath, Jana Noskova, Maksym Ivashechkin, and Jiri Matas.
\newblock Magsac++, a fast, reliable and accurate robust estimator.
\newblock In \emph{2020 {IEEE/CVF} Conference on Computer Vision and Pattern Recognition, {CVPR} 2020, Seattle, WA, USA, June 13-19, 2020}, 2020.

\bibitem[Barath et~al.(2021)Barath, Mishkin, Eichhardt, Shipachev, and Matas]{DBLP:conf/cvpr/BarathMESM21}
Daniel Barath, Dmytro Mishkin, Ivan Eichhardt, Ilia Shipachev, and Jiri Matas.
\newblock Efficient initial pose-graph generation for global sfm.
\newblock In \emph{CVPR}, 2021.

\bibitem[Bartoli and Sturm(2005)]{DBLP:journals/cviu/BartoliS05}
Adrien Bartoli and Peter~F. Sturm.
\newblock Structure-from-motion using lines: Representation, triangulation, and bundle adjustment.
\newblock \emph{CVIU}, 100\penalty0 (3), 2005.

\bibitem[Breiding et~al.(2022)Breiding, Rydell, Shehu, and Torres]{breiding2022line}
Paul Breiding, Felix Rydell, Elima Shehu, and Angélica Torres.
\newblock Line multiview varieties, 2022.

\bibitem[Camposeco et~al.(2018)Camposeco, Cohen, Pollefeys, and Sattler]{camposeco2018hybrid}
Federico Camposeco, Andrea Cohen, Marc Pollefeys, and Torsten Sattler.
\newblock Hybrid camera pose estimation.
\newblock In \emph{CVPR}, 2018.

\bibitem[Chen et~al.(2019)Chen, Han, Xu, and Su]{DBLP:conf/iccv/ChenHXS19}
Rui Chen, Songfang Han, Jing Xu, and Hao Su.
\newblock Point-based multi-view stereo network.
\newblock In \emph{ICCV}, 2019.

\bibitem[Chum et~al.(2003)Chum, Matas, and Kittler]{chum2003locally}
Ondrej Chum, Jiri Matas, and Josef Kittler.
\newblock Locally optimized ransac.
\newblock In \emph{Joint Pattern Recognition Symposium}, pages 236--243, 2003.

\bibitem[Dai et~al.(2017)Dai, Chang, Savva, Halber, Funkhouser, and Nie{\ss}ner]{dai2017scannet}
Angela Dai, Angel~X. Chang, Manolis Savva, Maciej Halber, Thomas Funkhouser, and Matthias Nie{\ss}ner.
\newblock Scannet: Richly-annotated 3d reconstructions of indoor scenes.
\newblock In \emph{CVPR}, 2017.

\bibitem[DeTone et~al.(2017)DeTone, Malisiewicz, and Rabinovich]{DBLP:journals/corr/DeToneMR17}
Daniel DeTone, Tomasz Malisiewicz, and Andrew Rabinovich.
\newblock Toward geometric deep {SLAM}.
\newblock \emph{CoRR}, 2017.

\bibitem[DeTone et~al.(2018)DeTone, Malisiewicz, and Rabinovich]{superpoint}
Daniel DeTone, Tomasz Malisiewicz, and Andrew Rabinovich.
\newblock {SuperPoint}: Self-supervised interest point detection and description.
\newblock In \emph{CVPR}, 2018.

\bibitem[Dosovitskiy et~al.(2015)Dosovitskiy, Fischer, Ilg, H{\"{a}}usser, Hazirbas, Golkov, van~der Smagt, Cremers, and Brox]{DBLP:conf/iccv/DosovitskiyFIHH15}
Alexey Dosovitskiy, Philipp Fischer, Eddy Ilg, Philip H{\"{a}}usser, Caner Hazirbas, Vladimir Golkov, Patrick van~der Smagt, Daniel Cremers, and Thomas Brox.
\newblock Flownet: Learning optical flow with convolutional networks.
\newblock In \emph{ICCV}, 2015.

\bibitem[Duff et~al.(2019)Duff, Kohn, Leykin, and Pajdla]{DBLP:conf/iccv/DuffKLP19}
Timothy Duff, Kathl{\'{e}}n Kohn, Anton Leykin, and Tom{\'{a}}s Pajdla.
\newblock {PLMP} - point-line minimal problems in complete multi-view visibility.
\newblock In \emph{ICCV}, 2019.

\bibitem[Duff et~al.(2020)Duff, Kohn, Leykin, and Pajdla]{DBLP:conf/eccv/DuffKLP20}
Timothy Duff, Kathl{\'{e}}n Kohn, Anton Leykin, and Tom{\'{a}}s Pajdla.
\newblock Pl\({}_{\mbox{1}}\)p - point-line minimal problems under partial visibility in three views.
\newblock In \emph{ECCV}, 2020.

\bibitem[Elqursh and Elgammal(2011)]{DBLP:conf/cvpr/ElqurshE11}
Ali Elqursh and Ahmed~M. Elgammal.
\newblock Line-based relative pose estimation.
\newblock In \emph{The 24th {IEEE} Conference on Computer Vision and Pattern Recognition, {CVPR} 2011, Colorado Springs, CO, USA, 20-25 June 2011}, pages 3049--3056. {IEEE} Computer Society, 2011.

\bibitem[Engel et~al.(2014)Engel, Sch{\"{o}}ps, and Cremers]{DBLP:conf/eccv/EngelSC14}
Jakob Engel, Thomas Sch{\"{o}}ps, and Daniel Cremers.
\newblock {LSD-SLAM:} large-scale direct monocular {SLAM}.
\newblock In \emph{ECCV}, 2014.

\bibitem[Fabbri et~al.(2020)Fabbri, Duff, Fan, Regan, da~Costa~de Pinho, Tsigaridas, Wampler, Hauenstein, Giblin, Kimia, Leykin, and Pajdla]{DBLP:conf/cvpr/FabbriDFRPTWHGK20}
Ricardo Fabbri, Timothy Duff, Hongyi Fan, Margaret~H. Regan, David da~Costa~de Pinho, Elias~P. Tsigaridas, Charles~W. Wampler, Jonathan~D. Hauenstein, Peter~J. Giblin, Benjamin~B. Kimia, Anton Leykin, and Tom{\'{a}}s Pajdla.
\newblock {TRPLP} - trifocal relative pose from lines at points.
\newblock In \emph{CVPR}, 2020.

\bibitem[Fischler and Bolles(1987)]{fischler1987}
Martin~A. Fischler and Robert~C. Bolles.
\newblock Random sample consensus: A paradigm for model fitting with applications to image analysis and automated cartography.
\newblock In \emph{Readings in Computer Vision}, 1987.

\bibitem[Fraundorfer et~al.(2010)Fraundorfer, Tanskanen, and Pollefeys]{DBLP:conf/eccv/FraundorferTP10}
Friedrich Fraundorfer, Petri Tanskanen, and Marc Pollefeys.
\newblock A minimal case solution to the calibrated relative pose problem for the case of two known orientation angles.
\newblock In \emph{ECCV}, 2010.

\bibitem[Furukawa and Hern{\'{a}}ndez(2015)]{DBLP:journals/ftcgv/FurukawaH15}
Yasutaka Furukawa and Carlos Hern{\'{a}}ndez.
\newblock Multi-view stereo: {A} tutorial.
\newblock \emph{FTCGV}, 9\penalty0 (1-2), 2015.

\bibitem[Furukawa et~al.(2010)Furukawa, Curless, Seitz, and Szeliski]{DBLP:conf/cvpr/FurukawaCSS10}
Yasutaka Furukawa, Brian Curless, Steven~M. Seitz, and Richard Szeliski.
\newblock Towards internet-scale multi-view stereo.
\newblock In \emph{CVPR}, 2010.

\bibitem[Gallier et~al.(2012)Gallier, Zhou, Naroditsky, Roumeliotis, and Daniilidis]{Gallier12}
J. Gallier, X.~S. Zhou, O. Naroditsky, S.~I. Roumeliotis, and K. Daniilidis.
\newblock Two efficient solutions for visual odometry using directional correspondence.
\newblock \emph{IEEE Transactions on Pattern Analysis \&amp; Machine Intelligence}, 34\penalty0 (04), 2012.

\bibitem[Gao et~al.(2021)Gao, Wan, Ping, Zhang, Dong, Li, and Guo]{gao2021pose}
Shuang Gao, Jixiang Wan, Yishan Ping, Xudong Zhang, Shuzhou Dong, Jijunnan Li, and Yandong Guo.
\newblock Pose refinement with joint optimization of visual points and lines.
\newblock \emph{arXiv preprint arXiv:2110.03940}, 2021.

\bibitem[Geiger et~al.(2012)Geiger, Lenz, and Urtasun]{Geiger2012CVPR}
Andreas Geiger, Philip Lenz, and Raquel Urtasun.
\newblock Are we ready for autonomous driving? the kitti vision benchmark suite.
\newblock In \emph{Conference on Computer Vision and Pattern Recognition (CVPR)}, 2012.

\bibitem[Geiger et~al.(2013)Geiger, Lenz, Stiller, and Urtasun]{Geiger2013IJRR}
Andreas Geiger, Philip Lenz, Christoph Stiller, and Raquel Urtasun.
\newblock Vision meets robotics: The kitti dataset.
\newblock \emph{International Journal of Robotics Research (IJRR)}, 2013.

\bibitem[Geppert et~al.(2020)Geppert, Larsson, Speciale, Sch{\"{o}}nberger, and Pollefeys]{Geppert2020ECCV}
Marcel Geppert, Viktor Larsson, Pablo Speciale, Johannes~L. Sch{\"{o}}nberger, and Marc Pollefeys.
\newblock Privacy preserving structure-from-motion.
\newblock In \emph{ECCV}, 2020.

\bibitem[Geppert et~al.(2021)Geppert, Larsson, Speciale, Sch{\"{o}}nberger, and Pollefeys]{Geppert2021CVPR}
Marcel Geppert, Viktor Larsson, Pablo Speciale, Johannes~L. Sch{\"{o}}nberger, and Marc Pollefeys.
\newblock Privacy preserving localization and mapping from uncalibrated cameras.
\newblock In \emph{CVPR}, 2021.

\bibitem[Gomez-Ojeda et~al.(2019)Gomez-Ojeda, Moreno, Zuniga-No{\"e}l, Scaramuzza, and Gonzalez-Jimenez]{gomez2019pl}
Ruben Gomez-Ojeda, Francisco-Angel Moreno, David Zuniga-No{\"e}l, Davide Scaramuzza, and Javier Gonzalez-Jimenez.
\newblock Pl-slam: A stereo slam system through the combination of points and line segments.
\newblock \emph{IEEE Transactions on Robotics}, 35\penalty0 (3), 2019.

\bibitem[Guerrero and Sagues(2001)]{guerrero2001lines}
JJ Guerrero and C Sagues.
\newblock From lines to homographies between uncalibrated images.
\newblock In \emph{IX Symposium on Pattern Recognition and Image Analysis, VO4}, pages 233--240, 2001.

\bibitem[Hartley and Zisserman(2004)]{DBLP:books/cu/HZ2004}
Richard Hartley and Andrew Zisserman.
\newblock \emph{Multiple View Geometry in Computer Vision}.
\newblock Cambridge University Press, 2004.

\bibitem[Heinly et~al.(2015)Heinly, Sch{\"{o}}nberger, Dunn, and Frahm]{DBLP:conf/cvpr/HeinlySDF15}
Jared Heinly, Johannes~L. Sch{\"{o}}nberger, Enrique Dunn, and Jan{-}Michael Frahm.
\newblock Reconstructing the world* in six days.
\newblock In \emph{CVPR}, 2015.

\bibitem[Horn and Schunck(1981)]{DBLP:journals/ai/HornS81}
Berthold K.~P. Horn and Brian~G. Schunck.
\newblock Determining optical flow.
\newblock \emph{Artif. Intell.}, 17\penalty0 (1-3), 1981.

\bibitem[Hruby et~al.(2022)Hruby, Duff, Leykin, and Pajdla]{DBLP:conf/cvpr/HrubyDLP22}
Petr Hruby, Timothy Duff, Anton Leykin, and Tom{\'{a}}s Pajdla.
\newblock Learning to solve hard minimal problems.
\newblock In \emph{{IEEE/CVF} Conference on Computer Vision and Pattern Recognition, {CVPR} 2022, New Orleans, LA, USA, June 18-24, 2022}, pages 5522--5532. {IEEE}, 2022.

\bibitem[Ivashechkin et~al.(2021)Ivashechkin, Barath, and Matas]{DBLP:conf/iccv/IvashechkinBM21}
Maksym Ivashechkin, Daniel Barath, and Jiri Matas.
\newblock {VSAC:} efficient and accurate estimator for {H} and {F}.
\newblock In \emph{ICCV}, 2021.

\bibitem[Jin et~al.(2020)Jin, Mishkin, Mishchuk, Matas, Fua, Yi, and Trulls]{Jin2020}
Yuhe Jin, Dmytro Mishkin, Anastasiia Mishchuk, Jiri Matas, Pascal Fua, Kwang~Moo Yi, and Eduard Trulls.
\newblock {Image Matching across Wide Baselines: From Paper to Practice}.
\newblock \emph{IJCV}, 2020.

\bibitem[Kalantari et~al.(2011)Kalantari, Hashemi, Jung, and Gu{\'{e}}don]{DBLP:journals/jmiv/KalantariHJG11}
Mahzad Kalantari, Amir Hashemi, Franck Jung, and Jean{-}Pierre Gu{\'{e}}don.
\newblock A new solution to the relative orientation problem using only 3 points and the vertical direction.
\newblock \emph{J. Math. Imaging Vis.}, 39\penalty0 (3), 2011.

\bibitem[Kar et~al.(2017)Kar, H{\"{a}}ne, and Malik]{DBLP:conf/nips/KarHM17}
Abhishek Kar, Christian H{\"{a}}ne, and Jitendra Malik.
\newblock Learning a multi-view stereo machine.
\newblock In \emph{NeurIPS}, 2017.

\bibitem[Kukelova et~al.(2008)Kukelova, Bujnak, and Pajdla]{DBLP:conf/bmvc/KukelovaBP08}
Zuzana Kukelova, Martin Bujnak, and Tom{\'{a}}s Pajdla.
\newblock Polynomial eigenvalue solutions to the 5-pt and 6-pt relative pose problems.
\newblock In \emph{BMVC}, 2008.

\bibitem[Levenberg(1944)]{lma}
Kenneth Levenberg.
\newblock A method for the solution of certain non-linear problems in least squares.
\newblock \emph{Quarterly of Applied Mathematics}, 2\penalty0 (2), 1944.

\bibitem[Li and Larsson(2020)]{DBLP:journals/corr/abs-2004-11765}
Bo Li and Viktor Larsson.
\newblock {GAPS:} generator for automatic polynomial solvers.
\newblock \emph{CoRR}, abs/2004.11765, 2020.

\bibitem[Li and Hartley(2006)]{DBLP:conf/icpr/LiH06}
Hongdong Li and Richard~I. Hartley.
\newblock Five-point motion estimation made easy.
\newblock In \emph{ICPR}, 2006.

\bibitem[Lim et~al.(2021)Lim, Kim, Jung, Hu, and Myung]{lim2021avoiding}
Hyunjun Lim, Yeeun Kim, Kwangik Jung, Sumin Hu, and Hyun Myung.
\newblock Avoiding degeneracy for monocular visual slam with point and line features.
\newblock In \emph{ICRA}, 2021.

\bibitem[Lowe(2004)]{SIFT2004}
David~G. Lowe.
\newblock Distinctive image features from scale-invariant keypoints.
\newblock \emph{IJCV}, 60\penalty0 (2), 2004.

\bibitem[Lynen et~al.(2020)Lynen, Zeisl, Aiger, Bosse, Hesch, Pollefeys, Siegwart, and Sattler]{DBLP:journals/ijrr/LynenZABHPSS20}
Simon Lynen, Bernhard Zeisl, Dror Aiger, Michael Bosse, Joel~A. Hesch, Marc Pollefeys, Roland Siegwart, and Torsten Sattler.
\newblock Large-scale, real-time visual-inertial localization revisited.
\newblock \emph{IJRR}, 39\penalty0 (9), 2020.

\bibitem[Mikolajczyk and Schmid(2004)]{HesHarAff2004}
Krystian Mikolajczyk and Cordelia Schmid.
\newblock Scale \& affine invariant interest point detectors.
\newblock \emph{IJCV}, 60\penalty0 (1), 2004.

\bibitem[Mishkin et~al.(2015)Mishkin, Matas, and Perdoch]{Mishkin2015MODS}
Dmytro Mishkin, Jiri Matas, and Michal Perdoch.
\newblock Mods: Fast and robust method for two-view matching.
\newblock \emph{CVIU}, 2015.

\bibitem[Mishkin et~al.(2018)Mishkin, Radenovic, and Matas]{AffNet2018}
Dmytro Mishkin, Filip Radenovic, and Jiri Matas.
\newblock {Repeatability is Not Enough: Learning Affine Regions via Discriminability}.
\newblock In \emph{ECCV}, 2018.

\bibitem[Mur{-}Artal et~al.(2015)Mur{-}Artal, Montiel, and Tard{\'{o}}s]{DBLP:journals/trob/Mur-ArtalMT15}
Raul Mur{-}Artal, J.~M.~M. Montiel, and Juan~D. Tard{\'{o}}s.
\newblock {ORB-SLAM:} {A} versatile and accurate monocular {SLAM} system.
\newblock \emph{{IEEE} Trans. Robotics}, 31\penalty0 (5), 2015.

\bibitem[Nist{\'{e}}r(2003)]{DBLP:conf/cvpr/Nister03}
David Nist{\'{e}}r.
\newblock An efficient solution to the five-point relative pose problem.
\newblock In \emph{CVPR}, 2003.

\bibitem[Nist{\'{e}}r et~al.(2004)Nist{\'{e}}r, Naroditsky, and Bergen]{DBLP:conf/cvpr/NisterNB04}
David Nist{\'{e}}r, Oleg Naroditsky, and James~R. Bergen.
\newblock Visual odometry.
\newblock In \emph{CVPR}, 2004.

\bibitem[Nist{\'{e}}r et~al.(2006)Nist{\'{e}}r, Naroditsky, and Bergen]{DBLP:journals/jfr/NisterNB06}
David Nist{\'{e}}r, Oleg Naroditsky, and James~R. Bergen.
\newblock Visual odometry for ground vehicle applications.
\newblock \emph{J. Field Robotics}, 23\penalty0 (1), 2006.

\bibitem[Panek et~al.(2022)Panek, Kukelova, and Sattler]{DBLP:conf/eccv/PanekKS22}
Vojtech Panek, Zuzana Kukelova, and Torsten Sattler.
\newblock Meshloc: Mesh-based visual localization.
\newblock In \emph{ECCV}, 2022.

\bibitem[Pautrat et~al.(2021)Pautrat, Lin, Larsson, Oswald, and Pollefeys]{Pautrat_Lin_2021_CVPR}
Rémi Pautrat, Juan-Ting Lin, Viktor Larsson, Martin~R. Oswald, and Marc Pollefeys.
\newblock {SOLD2}: Self-supervised occlusion-aware line description and detection.
\newblock In \emph{CVPR}, 2021.

\bibitem[Pautrat et~al.(2023{\natexlab{a}})Pautrat, Barath, Larsson, Oswald, and Pollefeys]{pautrat2023deeplsd}
R{\'e}mi Pautrat, Daniel Barath, Viktor Larsson, Martin~R Oswald, and Marc Pollefeys.
\newblock Deeplsd: Line segment detection and refinement with deep image gradients.
\newblock In \emph{Proceedings of the IEEE/CVF Conference on Computer Vision and Pattern Recognition}, pages 17327--17336, 2023{\natexlab{a}}.

\bibitem[Pautrat et~al.(2023{\natexlab{b}})Pautrat, Su{\'a}rez, Yu, Pollefeys, and Larsson]{pautrat2023gluestick}
R{\'e}mi Pautrat, Iago Su{\'a}rez, Yifan Yu, Marc Pollefeys, and Viktor Larsson.
\newblock Gluestick: Robust image matching by sticking points and lines together.
\newblock \emph{Internation Conference on Computer Vision}, 2023{\natexlab{b}}.

\bibitem[Pumarola et~al.(2017)Pumarola, Vakhitov, Agudo, Sanfeliu, and Moreno-Noguer]{pumarola2017pl}
Albert Pumarola, Alexander Vakhitov, Antonio Agudo, Alberto Sanfeliu, and Francese Moreno-Noguer.
\newblock Pl-slam: Real-time monocular visual slam with points and lines.
\newblock In \emph{ICRA}, 2017.

\bibitem[Raguram et~al.(2013)Raguram, Chum, Pollefeys, Matas, and Frahm]{DBLP:journals/pami/RaguramCPMF13}
Rahul Raguram, Ondrej Chum, Marc Pollefeys, Jiri Matas, and Jan{-}Michael Frahm.
\newblock {USAC:} {A} universal framework for random sample consensus.
\newblock \emph{IEEE TPAMI}, 35\penalty0 (8), 2013.

\bibitem[Revaud et~al.(2019)Revaud, de~Souza, Humenberger, and Weinzaepfel]{DBLP:conf/nips/RevaudSHW19}
J{\'{e}}r{\^{o}}me Revaud, C{\'{e}}sar~Roberto de Souza, Martin Humenberger, and Philippe Weinzaepfel.
\newblock {R2D2:} reliable and repeatable detector and descriptor.
\newblock In \emph{NeurIPS}, 2019.

\bibitem[Salaün et~al.(2016)Salaün, Marlet, and Monasse]{SalaunMM16}
Yohann Salaün, Renaud Marlet, and Pascal Monasse.
\newblock Robust and accurate line- and/or point-based pose estimation without manhattan assumptions.
\newblock In \emph{ECCV 2016}, 2016.

\bibitem[Sarlin et~al.(2020)Sarlin, DeTone, Malisiewicz, and Rabinovich]{sarlin20superglue}
Paul-Edouard Sarlin, Daniel DeTone, Tomasz Malisiewicz, and Andrew Rabinovich.
\newblock {SuperGlue}: Learning feature matching with graph neural networks.
\newblock In \emph{CVPR}, 2020.

\bibitem[Sarlin et~al.(2022)Sarlin, Dusmanu, Sch\"onberger, Speciale, Gruber, Larsson, Miksik, and Pollefeys]{sarlin2022lamar}
Paul-Edouard Sarlin, Mihai Dusmanu, Johannes~L. Sch\"onberger, Pablo Speciale, Lukas Gruber, Viktor Larsson, Ondrej Miksik, and Marc Pollefeys.
\newblock {LaMAR}: {B}enchmarking {L}ocalization and {M}apping for {A}ugmented {R}eality.
\newblock In \emph{ECCV}, 2022.

\bibitem[Sattler et~al.(2012)Sattler, Leibe, and Kobbelt]{DBLP:conf/eccv/SattlerLK12}
Torsten Sattler, Bastian Leibe, and Leif Kobbelt.
\newblock Improving image-based localization by active correspondence search.
\newblock In \emph{ECCV}, 2012.

\bibitem[Sattler et~al.(2018)Sattler, Maddern, Toft, Torii, Hammarstrand, Stenborg, Safari, Okutomi, Pollefeys, Sivic, Kahl, and Pajdla]{DBLP:conf/cvpr/SattlerMTTHSSOP18}
Torsten Sattler, Will Maddern, Carl Toft, Akihiko Torii, Lars Hammarstrand, Erik Stenborg, Daniel Safari, Masatoshi Okutomi, Marc Pollefeys, Josef Sivic, Fredrik Kahl, and Tom{\'{a}}s Pajdla.
\newblock Benchmarking 6dof outdoor visual localization in changing conditions.
\newblock In \emph{CVPR}, 2018.

\bibitem[Saurer et~al.(2012)Saurer, Fraundorfer, and Pollefeys]{20.500.11850/63639}
Olivier Saurer, Friedrich Fraundorfer, and Marc Pollefeys.
\newblock Homography based visual odometry with known vertical direction and weak manhattan world assumption, 2012.

\bibitem[Sch{\"{o}}nberger and Frahm(2016)]{DBLP:conf/cvpr/SchonbergerF16}
Johannes~L. Sch{\"{o}}nberger and Jan{-}Michael Frahm.
\newblock Structure-from-motion revisited.
\newblock In \emph{CVPR}, 2016.

\bibitem[Schops et~al.(2017)Schops, Schonberger, Galliani, Sattler, Schindler, Pollefeys, and Geiger]{Schops_2017_eth3d}
Thomas Schops, Johannes~L. Schonberger, Silvano Galliani, Torsten Sattler, Konrad Schindler, Marc Pollefeys, and Andreas Geiger.
\newblock A multi-view stereo benchmark with high-resolution images and multi-camera videos.
\newblock In \emph{CVPR}, 2017.

\bibitem[Shotton et~al.(2013)Shotton, Glocker, Zach, Izadi, Criminisi, and Fitzgibbon]{7scenes}
Jamie Shotton, Ben Glocker, Christopher Zach, Shahram Izadi, Antonio Criminisi, and Andrew Fitzgibbon.
\newblock Scene coordinate regression forests for camera relocalization in rgb-d images.
\newblock In \emph{CVPR}, 2013.

\bibitem[Shu et~al.(2022)Shu, Wang, Pagani, and Stricker]{shu2022structure}
Fangwen Shu, Jiaxuan Wang, Alain Pagani, and Didier Stricker.
\newblock Structure plp-slam: Efficient sparse mapping and localization using point, line and plane for monocular, rgb-d and stereo cameras.
\newblock \emph{arXiv}, 2022.

\bibitem[Snavely et~al.(2006)Snavely, Seitz, and Szeliski]{phototourism}
Noah Snavely, Steven~M. Seitz, and Richard Szeliski.
\newblock Photo tourism: Exploring photo collections in 3d.
\newblock In \emph{ACM SIGGRAPH Conference}, 2006.

\bibitem[Snavely et~al.(2008)Snavely, Seitz, and Szeliski]{DBLP:journals/ijcv/SnavelySS08}
Noah Snavely, Steven~M. Seitz, and Richard Szeliski.
\newblock Modeling the world from internet photo collections.
\newblock \emph{IJCV}, 80\penalty0 (2), 2008.

\bibitem[Stewénius et~al.(2006)Stewénius, Engels, and Nistér]{STEWENIUS2006284}
Henrik Stewénius, Christopher Engels, and David Nistér.
\newblock Recent developments on direct relative orientation.
\newblock \emph{ISPRS Journal of Photogrammetry and Remote Sensing}, 60\penalty0 (4):\penalty0 284--294, 2006.

\bibitem[Sun et~al.(2021)Sun, Shen, Wang, Bao, and Zhou]{sun2021loftr}
Jiaming Sun, Zehong Shen, Yuang Wang, Hujun Bao, and Xiaowei Zhou.
\newblock Loftr: Detector-free local feature matching with transformers.
\newblock In \emph{Proceedings of the IEEE/CVF conference on computer vision and pattern recognition}, pages 8922--8931, 2021.

\bibitem[Sweeney et~al.(2014)Sweeney, Flynn, and Turk]{DBLP:conf/3dim/SweeneyFT14}
Chris Sweeney, John Flynn, and Matthew~A. Turk.
\newblock Solving for relative pose with a partially known rotation is a quadratic eigenvalue problem.
\newblock In \emph{2nd International Conference on 3D Vision, 3DV 2014, Tokyo, Japan, December 8-11, 2014, Volume 1}, pages 483--490. {IEEE} Computer Society, 2014.

\bibitem[Tardif(2009)]{Tardif09}
Jean-Philippe Tardif.
\newblock Non-iterative approach for fast and accurate vanishing point detection.
\newblock In \emph{ICCV}, 2009.

\bibitem[Tyszkiewicz et~al.(2020)Tyszkiewicz, Fua, and Trulls]{DBLP:conf/nips/TyszkiewiczFT20}
Michal~J. Tyszkiewicz, Pascal Fua, and Eduard Trulls.
\newblock {DISK:} learning local features with policy gradient.
\newblock In \emph{NeurIPS}, 2020.

\bibitem[Wei et~al.(2019)Wei, Huang, and Ma]{wei2019real}
Xinyu Wei, Jun Huang, and Xiaoyuan Ma.
\newblock Real-time monocular visual slam by combining points and lines.
\newblock In \emph{ICME}, 2019.

\bibitem[Yoon and Kim(2021)]{syoon_2021_linetr}
Sungho Yoon and Ayoung Kim.
\newblock {Line as a Visual Sentence}: Context-aware line descriptor for visual localization.
\newblock \emph{IEEE Robotics and Automation Letters}, 2021.

\bibitem[Yu and Morel(2011)]{DBLP:journals/ipol/YuM11}
Guoshen Yu and Jean{-}Michel Morel.
\newblock {ASIFT:} an algorithm for fully affine invariant comparison.
\newblock \emph{Image Process. Line}, 1, 2011.

\bibitem[Zhang and Koch(2013)]{zhang2013lbd}
Lilian Zhang and Reinhard Koch.
\newblock An efficient and robust line segment matching approach based on lbd descriptor and pairwise geometric consistency.
\newblock \emph{Journal of Visual Communication and Image Representation}, 24, 2013.

\bibitem[Zhu et~al.(2018)Zhu, Zhang, Zhou, Shen, Fang, Tan, and Quan]{DBLP:conf/cvpr/ZhuZZSFTQ18}
Siyu Zhu, Runze Zhang, Lei Zhou, Tianwei Shen, Tian Fang, Ping Tan, and Long Quan.
\newblock Very large-scale global sfm by distributed motion averaging.
\newblock In \emph{CVPR}, 2018.

\bibitem[Zuo et~al.(2017)Zuo, Xie, Liu, and Huang]{zuo2017robust}
Xingxing Zuo, Xiaojia Xie, Yong Liu, and Guoquan Huang.
\newblock Robust visual slam with point and line features.
\newblock In \emph{IROS}, 2017.

\end{thebibliography}
}

\end{document}